\documentclass{article}

\usepackage[preprint]{neurips_2025}

\usepackage[utf8]{inputenc}
\usepackage[T1]{fontenc}   
\usepackage[
colorlinks=true,
linkcolor=blue,
urlcolor=blue,
citecolor=blue
]{hyperref}
\usepackage{url}           
\usepackage{booktabs}       %
\usepackage{nicefrac}       %
\usepackage{microtype}      %
\usepackage{xcolor}         %

\usepackage{amsmath, bm}
\usepackage{amssymb, stmaryrd}

\usepackage{amsthm}
\usepackage{mathtools}
\usepackage{mathrsfs}
\usepackage{mathabx}
\usepackage{dsfont}
\usepackage{mathtools}
\usepackage{stmaryrd}
\usepackage{xfrac}
\usepackage{derivative}

\usepackage{booktabs,tabularx, multirow}
\usepackage{titletoc}
\newcommand{\printappendixtoc}{%
  \startcontents[appendix]%
  \printcontents[appendix]{l}{1}{\section*{Contents of the Appendix}}{}%
}
\usepackage{subcaption}
\usepackage{wrapfig}
\usepackage{xspace}
\usepackage{duckuments} %
\usepackage{pgfplots}
\usepackage{scalefnt}
\pgfplotsset{compat=1.18}
\usepackage{graphicx}
\usepackage[export]{adjustbox}
\usepackage[rawfloats=true]{floatrow}
\usepackage[outline]{contour} %

\usepackage{tikz}
\usepackage{tikz-3dplot}
\tdplotsetmaincoords{70}{110}   %

\usepackage{listofitems}
\usetikzlibrary{fpu}
\usetikzlibrary{calc} %
\usepackage[skins]{tcolorbox} %
\usetikzlibrary{shadows.blur}
\usetikzlibrary{fit}
\usepackage{scalefnt}
\usepackage{mwe}
\usetikzlibrary{shapes.geometric}
\usetikzlibrary{patterns, positioning,fit}

\tcbuselibrary{theorems}
\newtcbtheorem
[]%
{propertybox}%
{Property}%
{%
	colback=gray!1,
	colframe=gray!35!black,
	fonttitle=\bfseries,
}%
{propbox}%

\title{Feature Attribution from First Principles}

\author{%
  Magamed~Taimeskhanov \\
  Julius-Maximilians Universit\"at W\"urzburg, CAIDAS \\
   W\"urzburg, Germany \\
  \texttt{magamed.taimeskhanov@uni-wuerzburg.de} \\
  \And
  Damien Garreau \\
  Julius-Maximilians Universit\"at W\"urzburg, CAIDAS \\
  W\"urzburg, Germany \\
  \texttt{damien.garreau@uni-wuerzburg.de} \\
}

\DeclareMathOperator*{\Argmax}{arg\,max}%
\DeclareMathOperator{\Diag}{diag}%
\DeclareMathOperator{\Expec}{\mathbb{E}}%

\newcommand{\abs}[1]{\left\lvert#1\right\rvert}%
\newcommand{\card}[1]{\left\lvert#1\right\rvert} %
\newcommand{\condexpec}[2]{\mathbb{E}\left[#1\middle|#2\right]}%

\newcommand{\diag}[1]{\Diag\left(#1\right)}%
\newcommand{\Diff}{\mathrm{d}}%
\newcommand{\defeq}{\vcentcolon =}%
\newcommand{\expecunder}[2]{\Expec_{#2}\left[#1\right]}%
\newcommand{\Indic}{\mathds{1}}%
\newcommand{\indic}[1]{\Indic_{#1}}%
\newcommand{\Inputspace}{\mathcal{X}}%
\newcommand{\norm}[1]{\left\lVert#1\right\rVert}%
\newcommand{\Posint}{\mathbb{N}^{\star}}%
\newcommand{\Reals}{\mathbb{R}}%

\newcommand{\Img}{\bm{\xi}} %

\newcommand{\ps}[2]{\left\langle #1 \; , \; #2 \right\rangle}

\newcommand{\xbf}{\mathbf{x}}
\newcommand{\Xbf}{\mathbf{X}}
\newcommand{\wbf}{\mathbf{w}}
\newcommand{\bbf}{\mathbf{b}}
\newcommand{\abf}{\mathbf{a}}

\newcommand{\Wbf}{\mathbf{W}}

\newcommand{\nbf}{\mathbf{n}}
\newcommand{\mbf}{\mathbf{m}}
\newcommand{\vbf}{\mathbf{v}}
\newcommand{\zbf}{\mathbf{z}}
\newcommand{\ybf}{\mathbf{y}}

\definecolor{ao}{rgb}{0.0, 0.5, 0.0}

\renewcommand{\epsilon}{\varepsilon}
\newcommand{\epsilonbm}{\bm{\epsilon}}

\theoremstyle{plain}
\newtheorem{theorem}{Theorem}[section]
\newtheorem{proposition}{Proposition}[section]
\newtheorem{lemma}{Lemma}[section]
\newtheorem{corollary}{Corollary}[section]

\theoremstyle{definition}

\newtheorem{definition}{Definition}[section]
\newtheorem{remark}{Remark}[section]

\makeatletter
\def\hlinewd#1{%
	\noalign{\ifnum0=`}\fi\hrule \@height #1 %
	\futurelet\reserved@a\@xhline}
\makeatother

\makeatletter
\def\th@plain{%
	\thm@notefont{}%
	\itshape %
}
\def\th@definition{%
	\thm@notefont{}%
	\normalfont %
}
\makeatother

\begin{document}

\maketitle

\begin{abstract}
Feature attribution methods are a popular approach to explain the behavior of machine learning models. 
They assign importance scores to each input feature, quantifying their influence on the model's prediction. 
However, evaluating these methods empirically remains a significant challenge. 
To bypass this shortcoming, several prior works have proposed axiomatic frameworks that any feature attribution method should satisfy. 
In this work, we argue that such axioms are often too restrictive, and propose in response a new feature attribution framework, built from the ground up. 
Rather than imposing axioms, we start by defining attributions for the simplest possible models, \emph{i.e.}, indicator functions, and use these as building blocks for more complex models. 
We then show that one recovers several existing attribution methods, depending on the choice of atomic attribution. 
Subsequently, we derive closed-form expressions for attribution of deep ReLU networks, and take a step toward the optimization of evaluation metrics with respect to feature attributions.
\end{abstract}

\section{Introduction}

As machine learning models continue to improve, they are increasingly deployed in critical domains such as healthcare, self-driving, etc.
Despite their impressive performances, deep learning models remain notoriously opaque due to their vast parameter spaces, complex designs, and reliance on a succession of non-linear transformations.
These characteristics hinder human understanding of how specific outputs are produced, earning such models the label of ``black boxes'' \citep{benitez_et_al_1997}.
In such settings, lack of transparency can undermine trust, hinder error diagnosis, and raise ethical or legal concerns.
In response, the field of eXplainable AI (XAI) has emerged, aiming to bridge this gap by making model behavior more transparent and interpretable.

A wide range of interpretability methods have emerged in recent years, such as LIME~\citep{ribeiro2016should}, KernelSHAP~\citep{lundberg2017unified}, Integrated Gradients~\citep{sundararajan2017axiomatic}, Grad-CAM~\citep{selvaraju2017grad} and LRP~\citep{bach2015pixel}. 
These approaches typically fall under the classical feature attribution paradigm, which seeks to assign importance scores to input features based on their contribution to a model's prediction. 
In computer vision, such feature attributions are often referred to as saliency maps.
Despite their popularity, these methods produce inconsistent attributions for the same model and input point to explain~\citep{hooker2019benchmark}, and no widely accepted theoretical foundation has yet emerged. 
While some efforts~\citep{sundararajan2017axiomatic, han2022explanation, bressan2024theory} have attempted to formalize attribution through axiomatic frameworks, a consensus on what constitutes a ``correct'' explanation remains elusive.

This work builds on the classical result that functions can be approximated by linear combinations of weighted indicator functions. 
We propose to begin by assigning atomic attributions to these indicator functions and then extend the attribution to more complex models through summation. 
Under mild regularity conditions, this construction yields a well-defined attribution method. 
This perspective leads naturally to a measure-theoretic framework for feature attribution, where explanations are represented as integrals. 
Our \textbf{main contributions} are threefold:
\begin{itemize}
    \item We analyze limitations of existing attribution axioms (Th.~\ref{th:bound-1}).
    \item We propose a constructivist, measure-theoretic framework for feature attribution (Th.~\ref{th:measure-atomic-attribution}).
    \item Within this framework, we derive closed-form attributions for deep ReLU networks (Cor.~\ref{cor:linear-positive}), recover existing methods (Table~\ref{tab:example-feature-attributions}), and show how to optimize evaluation metrics as functions of attributions (Th.~\ref{th:optim-1}).
\end{itemize}
Our implementation for computing these explanations is available \href{https://github.com/MagamedT/feature-attribution-from-first-principles}{online}. %
All proofs can be found in Appendix~\ref{app:proofs-main}.

\subsection{Related work}

\textbf{Empirical analysis of XAI methods.} 
Several empirical studies have revealed significant limitations of XAI methods, questioning in particular the reliability of saliency maps as faithful explanations~\citep{kindermans2019reliability, ghorbani2019interpretation}. 
Notably, \citep{adebayo2018sanity} introduced a randomization-based sanity check, showing that some XAI methods for images are independent of both the model and the data.
From a different perspective, \citep{heo2019fooling} directly attacked the reliability of saliency-based attributions by adversarial model manipulation, fine-tuning a model with the purpose of making the saliency maps unreliable.
Building upon these insights, several recent works have developed new evaluation frameworks.
\citep{hooker2019benchmark} and \citep{rong2022consistent} proposed several evaluation strategies to assess the quality of feature attributions and developed ranking schemes to compare the methods.
\citep{agarwal2022openxai} and \citep{hedstrom2023quantus} provide practical toolkits for systematically assessing the quality of feature attribution methods, aiming to streamline and standardize the evaluation of XAI methods.

\textbf{Theoretical analysis of XAI methods.} 
From the theoretical side, specific methods such as LIME~\citep{garreau2021does}, attention-based explanations~\citep{lopardo2024attention}, and KernelSHAP~\citep{bhattacharjee2025safely} have been analyzed, along with other work~\citep{ancona2017towards} covering additional methods such as DeepLIFT and $\epsilon$-LRP. 
Some works have taken a more general approach, presenting impossibility results under certain assumptions, for instance, regarding the feasibility of counterfactual recourse and robustness~\citep{fokkema2023attribution}, or the soundness of feature attributions under specific axioms~\citep{bilodeau2024impossibility}. 
More precisely, \citep{bilodeau2024impossibility} shows that, if the underlying assumptions are too general, one cannot conclude using the classical local feature attribution methods if a feature is important for the model's prediction when its attribution is largely positive.
\citep{konig2024disentangling} provides a decomposition for global feature attribution that disentangles a feature's individual contribution from those arising through interactions with other features.

\textbf{Axiomatisation of feature attribution.} 
Several works have proposed approaches to either unify or provide a solid foundation for XAI. 
\citep{covert2021explaining} introduces explainability through a feature removal lens, aiming to measure how much each feature contributes to the model's output. 
\citep{sundararajan2017axiomatic} proposes a set of axioms that any good explanation method should ideally satisfy. 
In the same spirit, \citep{guptaetal2022} introduces two new axioms, inspired by principles from logic, that any saliency-based method should satisfy, along with evaluation metrics designed to capture both properties.
\citep{han2022explanation} shows that many existing methods can be viewed as performing local approximations of the model being explained. 
Finally, \citep{bressan2024theory} also adopts an approximation-based view, framing interpretability in terms of PAC-learning.

\section{Limitations of existing axioms}
\label{sec:limitations}

In this section, we explore the landscape of existing axioms for feature attribution methods and show formally that they tend to produce similar attributions, as observed by~\citep{ancona2017towards}.
The main difference with~\citep{ancona2017towards} is that our result is valid for any feature attribution satisfying the axioms introduced in~\citep{sundararajan2017axiomatic}.

\subsection{Existing axioms}\label{sec:limitations.1}

The problem at hand is to explain the predictions of a function $f : \mathcal{X} \to \Reals$, with $\mathcal{X}$ the input space, and $\mathcal{F}$ is a vector space of functions to be explained.
One common approach is to use a feature attribution method $\phi$, defined as follows:

\begin{definition}[Feature attribution \citep{bilodeau2024impossibility}] 
A feature attribution method is any mapping $\phi : \mathcal{X} \times \mathcal{F} \to \Reals^d$, where $\phi(\xbf, f)_j$ represents the attribution of the $j$-th feature of the input $\xbf \in \mathcal{X}$ in explaining the output $f(\xbf)$, with $j \in \llbracket d \rrbracket \defeq \{1,2, \ldots, d \}$, and $f \in \mathcal{F}$.
\end{definition}

Intuitively, if $\phi(\xbf, f)_j \gg 0$, then the feature $\xbf_j$ is important for the model's prediction $f(\xbf)$.
We now list a handful of properties/axioms, found in the literature~\citep{sundararajan2017axiomatic}, which are often taken as desirable for any attribution method~$\phi$. 

\begin{definition}[Completeness~\citep{sundararajan2017axiomatic}]
\label{propbox:completeness}
Given a baseline $\xbf' \in \Reals^d$, we say that a feature attribution method $\phi$ is $\xbf'$-\texttt{complete} if:
\begin{align*}
\forall f \in \mathcal{F}, \forall \xbf \in \Reals^d, \qquad f(\xbf) - f(\xbf') = \sum_{j \in \llbracket d \rrbracket} \phi(\xbf, f)_j  \, .
\end{align*}
\end{definition}

That is, $\phi$ splits the model's output for any given input into contributions from each individual feature.
In the literature~\citep{bilodeau2024impossibility, sundararajan2017axiomatic}, \texttt{Completeness} is often defined so that it holds for every baseline $ \xbf' \in \mathbb{R}^d $. 
We do not consider this setting here. 

\begin{definition}[Sensitivity~\citep{sundararajan2017axiomatic}]\label{propbox:sensibility}
Define the perturbed vector $\mathbf{x}^{(x',j)} \in \mathbb{R}^d$ as the vector obtained by replacing the $j$-th coordinate of $\xbf$ with $x' \in \mathbb{R}$, \emph{i.e.}, $\xbf^{(x',j)}_j \defeq x'$.
A feature attribution $\phi$ is \texttt{sensitive} if:
	\begin{align*}
		& \forall f \in \mathcal{F}, \forall j \in \llbracket d \rrbracket, \quad \left(  \forall \xbf \in \Reals^d, \forall x' \in \Reals , \, f(\xbf) = f(\xbf^{(x',j)}) \implies  \forall \Img \in \Reals^d, \, \phi(\Img, f)_j = 0 \right) \, .
	\end{align*}
\end{definition}

In other words, if feature $j$ has no impact on the model's output, then $\phi_j=0$ for any input. 

\begin{definition}[Linearity~\citep{sundararajan2017axiomatic}]\label{propbox:linearity}
	A feature attribution $\phi$ is \texttt{linear} if:
	\begin{align*}
		\forall f, g \in \mathcal{F}, \forall \lambda \in \Reals, \forall \xbf \in \Reals^d, \qquad \phi(\xbf, f + \lambda g) =  \phi(\xbf,f) + \lambda \phi(\xbf, g) \, .
	\end{align*}
\end{definition}

We offer a more exhaustive review of existing axioms in Appendix~\ref{app:2.1}.

\subsection{Feature attribution derived from the previous axioms}\label{sec:limitations.2}

In this section, our goal is to show that any feature attribution method satisfying the previous axioms yields attributions similar to those of \emph{Gradient}$\times$\emph{Input}~\citep{shrikumar2016not}.
Given a differentiable model $f : \Reals^d \to \Reals$ and input to be explained $\xbf \in \Reals^d$, it is defined as
\[
\phi(\xbf, f) \defeq \nabla f(\xbf) \odot \xbf \, ,
\]
where $\odot$ denotes the element-wise product.
Intuitively, it captures both how large the feature is and how sensitive the prediction is. 
In the spirit of~\citep{ancona2017towards}, we now show that any feature attribution $\phi$ which satisfies \texttt{Completeness}, \texttt{Sensitivity}, and \texttt{Linearity} is a perturbation of the feature attribution method \emph{Gradient$\times$Input}, provided that the model is smooth enough. 
The magnitude of the perturbation is bounded by the ``degree of linearity'' (\emph{i.e.}, norm of the Hessian) of the explained model $f$. 
First, we show that the \texttt{Completeness} and \texttt{Sensitivity} properties are very strong as they can completely determine the attribution of some models. 

\begin{proposition}[Completeness $+$ Sensitivity]
\label{prop:sens_complete}
Define $\pi_j : \Reals^d \to \Reals$ the projection of $\xbf \in \Reals^d$ on its $j$-th coordinate as $\pi_j(\xbf) = \xbf_j$. 
Given a baseline $\xbf' \in \Reals^d$, a \texttt{sensitive} and $\xbf'$-\texttt{complete} feature attribution $\phi$ satisfies:
	\begin{align*}
		\forall f : \Reals \to \Reals, \forall \xbf \in \Reals^d, \forall k, j \in \llbracket d \rrbracket, \qquad  \phi\left(\xbf, f \circ \pi_j\right)_k = \indic{k=j} \left(f(\xbf_j) - f(\xbf_j')\right)    \, .
	\end{align*}
\end{proposition}

By adding the \texttt{Linearity} property into the previous proposition, one can show that any such feature attribution methods is equal to \emph{Gradient$\times$Input} plus a remainder.

\begin{lemma}[Completeness $+$ Sensitivity $+$ Linearity $+$ smoothness]\label{lemma:taylor-phi}
	Let $C^2(\Reals^d)$ be the space of twice continuously differentiable function from $\Reals^d$ to $\Reals$. 
    Given a vector $\xbf_0 \in \Reals^d$, a baseline $\xbf' \in \Reals^d$, and a \texttt{sensitive}, $\xbf'$-\texttt{complete}, \texttt{linear} feature attribution $\phi$, the following holds:
	\begin{align*}
		\forall f \in C^2(\Reals^d), \forall \xbf \in \Reals^d, \qquad  \phi(\xbf, f) = \nabla f (\xbf_0) \odot (\xbf - \xbf') + \phi(\xbf, R_{\xbf_0} )    \, ,
	\end{align*}
	with $R_{\xbf_0}(\zbf) \defeq \frac{1}{2} (\zbf - \xbf_0)^\top \nabla^2 f \left( \ybf_\zbf \right) (\zbf- \xbf_0)$ for $\zbf \in \Reals^d$ and, given $\zbf$, there exists $\tau \in [0,1]$ such that $\ybf_\zbf \defeq \xbf_0 + \tau(\zbf-\xbf_0)$. 
\end{lemma}

The remainder term $R_{\xbf_0}$ comes from the Taylor expansion of $f$. Since it is quadratic, the explanation of the remainder term $\phi(\xbf, R_{\xbf_0})$ may be large. 
We show how to control it under the assumptions that $\phi$ is Lipschitz continuous with respect to its second argument, and that the input space $\mathcal{X}$ is restricted to~$[0,1]^d$---this is a standard assumption in practice. 
Before proceeding, we specify Lipschitz continuity in our context.

\begin{definition}[Lipschitz continuity]
Set $B([0,1]^d)$ the space of bounded function from $[0,1]^d$ to $\Reals$. 
We say that $\phi : [0,1]^d \times B([0,1]^d) \to \Reals^d$ 
is Lipschitz continuous (in its second argument) if:
	\begin{align*}
		\forall \xbf \in [0,1]^d, \exists L_\xbf \geq 0,  \forall f, g \in B([0,1]^d), \qquad \norm{\phi(\xbf, f) - \phi(\xbf, g)}_2 \leq L_\xbf \norm{f-g}_\infty \, ,
	\end{align*}
where $\norm{\cdot}_2$ (resp. $\norm{\cdot}_\infty$) denotes the  Euclidean (resp. sup) norm.
\end{definition}

\begin{theorem}[Bound on remainder] 
\label{th:bound-1}
Define $M \defeq  \max_{\xbf \in [0,1]^d} \norm{\nabla^2 f (\xbf)}_{\mathrm{op}}$.
Given a baseline $\xbf' \in [0,1]^d$ and a \texttt{sensitive}, $\xbf'$-\texttt{complete}, \texttt{linear}, Lipschitz continuous feature attribution $\phi$, the following holds:
\begin{align*}
\forall f \in C^2([0,1]^d), \forall \xbf_0 \in [0,1]^d,  \forall \xbf \in [0,1]^d, \qquad  \norm{\phi(\xbf,  R_{\xbf_0})}_2 \leq \frac{L_\xbf  d}{2} M    
\, ,
\end{align*}
where $R_{\xbf_0}$ is as in Lemma~\ref{lemma:taylor-phi}, and $\norm{\cdot}_{\mathrm{op}}$ is the operator norm of a matrix.
\end{theorem}

Intuitively, $M$ measures the linearity of $f$ on $[0,1]^d$. 
If $f$ is linear, then $M = 0$ and we get no remainder, meaning that our attribution method $\phi$ on this model is exactly \emph{Gradient$\times$Input}.
\citep{ancona2017towards} and \citep{hesse2021fast} show, both theoretically and experimentally, that for deep ReLU networks, common attribution methods (Integrated Gradients, DeepLIFT, etc.) behave like \emph{Gradient$\times$Input}. 
\textbf{Th.~\ref{th:bound-1} generalizes this observation to arbitrary feature attribution methods checking the previous axioms.}

This theorem can be extended to deep ReLU networks by approximating them via convolution using a sequence of mollifier functions. 
More precisely, given a sequence $(\varphi_n)_{n \geq 0}$ of mollifiers functions, the smoothed network $\varphi_n \star f$ converges to the original network $f$ as $n \to +\infty$~\citep{Magyar1992Continuous}, where~$\star$ denotes convolution. 
Since $\varphi_n \star f$ is smooth, it satisfies the regularity conditions required by Th.~\ref{th:bound-1}. 
Moreover, under mild regularity assumptions on the attribution method $\phi$, the attribution $\phi(\mathbf{x}, \varphi_n \star f)$ converges to $\phi(\mathbf{x}, f)$ for any input $\xbf \in \mathbb{R}^d$.
This connection may allow the issues identified with \emph{Gradient$\times$Input}~\citep{adebayo2018sanity}, such as its tendency to behave like an edge detector, to be transferred to any such attribution methods.

Thus, \texttt{Completeness} and \texttt{Sensitivity}, while intuitive, are overly restrictive. 
Introducing any new element or even reinforcing/weakening existing axioms seems to inevitably undermine the axiomatic system's coherency. 
For instance, in Appendix~\ref{app:2.2}, we show that enforcing \texttt{Completeness} without a baseline leads to a contradiction. 
Similarly, formulating a local version of \texttt{Sensitivity} seems challenging, as naive formulations also lead to contradiction. 
Therefore, we propose an alternative approach, aiming to escape these pitfalls.
Among the previous axioms, we will only keep \texttt{Linearity}, as it is a reasonable property for any attribution method to satisfy.

\section{Our framework: a constructivist approach}\label{sec:3}
\label{sec:constructivist}

In our framework, we start with a simple model for which we can be confident in the explanation. 
A natural candidate for such a model is the indicator function.
This choice is justified by a classical approximation result: any continuous function defined on a compact set can be approximated arbitrarily well by a finite sum of indicator functions (see Appendix~\ref{app:3.1}).
Building on this, we derive attributions for the full model by aggregating the explanations of its indicator function components.
From now on, we restrict ourselves to models which are continuous on 
$\mathcal{X} \defeq [0,1]^d$.

\subsection{\texorpdfstring{Warm-up: $2$-dimensional example of our framework}{Warm-up: 2-dimensional example of our framework}}
\label{sec:3.1}
In order to give a good intuition of our framework, let us carry out the informal derivation in the case where the input dimension is $d = 2$.

\textbf{First step (Attribution of indicator functions).} 
Let $\indic{R} : [0,1]^2 \to \{0,1\}$ be the indicator function of a rectangle $R \defeq (a,b]\times (c,d] \subset [0,1]^2$ defined as $\indic{R}(\xbf) \defeq 1$ if $\xbf \in R$, else $0$.
A possible definition for the \emph{elementary (atomic) attribution} $\phi\left(\xbf,  \indic{R} \right)$ is
\[
\phi\left(\xbf, \indic{(a,b]\times (c,d]} \right) \defeq \left( \indic{\xbf_1 \in (a,b]}, \indic{\xbf_2 \in (c,d]} \right)^\top \in\Reals^2
\, .
\]
We argue that this definition is natural, since \textbf{a feature is considered positively important for the prediction of the model if it lies within the region where the indicator function is active.}  
Specifically, if $\xbf_1$ falls in the interval $(a,b]$, we give to feature $1$ a score of $1$.
The same reasoning applies to $\xbf_2$ with the interval $(c,d]$. 
To avoid binary attributions lacking a notion of magnitude, such as simply assigning $0$ or $1$, we scale each attribution by the length of the remaining intervals:
\begin{align*}
	\phi\left(\xbf, \indic{(a,b]\times (c,d]} \right) \defeq \left( (d-c) \indic{\xbf_1 \in (a,b]}, (b-a) \indic{\xbf_2 \in (c,d]} \right)^\top \, .
\end{align*}
The previous display can be interpreted as follows: if $(c,d]$ is significantly larger than $(a,b]$, then most of the attribution magnitude is assigned to the first feature. 
This makes sense because when $(c,d]$ covers a large range of the second coordinate, the condition $\xbf_1\in(a,b]$ is much more important as an explanation to the prediction $\indic{(a,b]\times(c,d]}$ than the condition 
$\xbf_2\in(c,d]$, which is always satisfied as $(c,d]$ covers a large portion of the second feature. 

\textbf{Second step (Approximation of the model).}
We approximate the function $f$ using a sequence of piecewise-constant functions $(f_n)_{n \geq 1}$ defined as:
\begin{equation}
\label{eq:approximation-step}
\forall n \in \Posint, \qquad f_n \defeq \sum_{i_1, i_2= 0}^{n-1} f\left(a^{(n)}_{i_1}, a^{(n)}_{i_2} \right) \indic{R_{i_1, i_2}^{(n)}} \, , 
\end{equation}
with $\left(a^{(n)}_{i} \right)_{i \in \llbracket 0, n \rrbracket}$ a partition ($n+1$ elements) of $[0,1]$, such that $a^{(n)}_{0} = 0 < a^{(n)}_{1} < \cdots < a^{(n)}_{n} = 1$, and the rectangles
\[
\forall n \in \Posint, \forall i,j \in \llbracket 0, n-1 \rrbracket, \qquad R_{i, j}^{(n)} \defeq \left(a^{(n)}_{i}, a^{(n)}_{i+1} \right] \times \left(a^{(n)}_{j}, a^{(n)}_{j+1} \right] \, .
\]
Standard arguments show that the sequence $(f_n)_{n \geq 1}$ is a good approximation of the model $f$ (see Appendix~\ref{app:3.1} for details). 
We depict this approximation in Figure~\ref{fig:approximation}.

\textbf{Third step (Attribution of sum of indicator functions).} 
Given $\xbf \in [0,1]^2$, if $\phi$ is \texttt{Linear}, we have
\begin{align}
\label{eq:rs-integral-d2}
\begin{split}
\phi\left(\xbf,  f_n  \right)_1 = \sum_{i_1, i_2= 0}^{n-1} f\left(a^{(n)}_{i_1}, a^{(n)}_{i_2} \right) \left( a^{(n)}_{i_2+1}-a^{(n)}_{i_2} \right) \left( \indic{\xbf_1 \leq a^{(n)}_{i_1+1}} - \indic{\xbf_1 \leq a^{(n)}_{i_1}}  \right) \, .
\end{split}
\end{align}

\textbf{Last step (Attribution of the whole model $f$).}
Under mild assumptions, the previous display converges to the following double integral:
\begin{align}
\begin{split} \label{eq:motivation-lim-1}
\forall \xbf \in (0,1]^2, \qquad  \int_{0}^1 \int_{0}^1 f(x,y) \, \Diff (\indic{x \geq \xbf_1}) \Diff y 
=  \int_{0}^1 f(\xbf_1, y) \, \Diff y \, .
\end{split}
\end{align}
The same argument applies to the second coordinate of the attribution. 
Finally, assuming minimal regularity assumptions on $\phi$, the sequence $\phi\left(\xbf, f_n \right)$ converges to $\phi(\xbf, f)$ as $n \to +\infty$. 
Combining this convergence with Equation~\eqref{eq:motivation-lim-1}, we obtain the feature attribution for the whole model $f$:
\begin{align}
\label{eq:pdp-2d}
\forall \xbf \in (0,1]^2, \qquad  \phi(\xbf, f)  = \left( \int_{0}^1 f(\xbf_1, y) \, \Diff y,  \int_{0}^1 f(x, \xbf_2) \, \Diff x\right)^\top 
\, . 
\end{align}
Intuitively, this means that each feature is explained by marginalizing the effect of the remaining feature, \emph{i.e.}, by integrating over it.
As it turns out, this coincides with the Partial Dependence Plot (PDP) of \citep{friedman2001greedy} under the uniform distribution, meaning no prior over the input space is assumed and all points are weighted equally.

\subsection{Motivation for our construction from a functional analysis perspective} 
\label{sec:3.2}

We can recover the previous integral form of feature attribution using a classical result from functional analysis, allowing us to bypass the more involved derivations presented earlier.
The construction from Section~\ref{sec:3.1} relies on the following regularity property of $\phi$:

\begin{definition}[Functional Supremum Continuity]
\label{propbox:fsc}
A feature attribution $\phi$ is \emph{functionally supremum continuous} (\texttt{FSC}) if for all sequences $(f_n)_{n \geq 0}$ ($f_n \in \mathcal{F}$) and all $f \in \mathcal{F}$:
\begin{align*}
 \lim_{n \to +\infty} \norm{f_n-f}_\infty = 0 \implies  \forall \xbf \in \Reals^d,  \lim_{n \to +\infty} \phi (\xbf, f_n) =  \phi (\xbf, f)  \, .
\end{align*}
\end{definition}

In other words, for all $\xbf \in \Reals^d$ and $j \in \llbracket d \rrbracket$, $\phi(\xbf, \cdot)_j$ is continuous where $\mathcal{F}$ is endowed with the topology of uniform convergence.  
\texttt{FSC} is also similar to the \texttt{Robustness} property~\citep{fokkema2023attribution} (see Appendix~\ref{app:2.1}). 
However, the difference is that this property provides stability to the attribution $\phi$ with respect to perturbations of the model $f \in \mathcal{F}$, rather than the explained point $\mathbf{x} \in \mathcal{X}$.
This new property enables us to construct a framework that offers greater flexibility than the axiomatic framework of~\citep{sundararajan2017axiomatic} (Section~\ref{sec:3.1}), without sacrificing a controlled structure for feature attribution.
One motivation for our constructivist approach, taken from functional analysis, is the Riesz–Markov representation theorem which we apply on any \texttt{Linear} and \texttt{FSC} attribution method $\phi$ as follows:

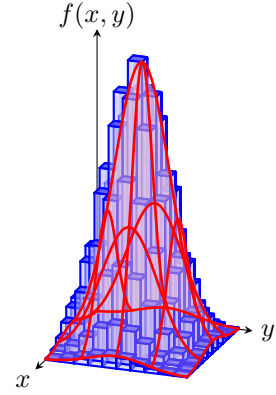
\begin{wrapfigure}{h}{0.4\textwidth}
  \centering
  \begin{tikzpicture}[tdplot_main_coords, line cap=round, line join=round,>=stealth, scale=2,
      declare function={
        f(\x,\y)=2*exp(-((\x-0.5)^2 + (\y-0.5)^2)/(2*0.2^2));
      }
  ]
    \draw[->] (0,0,0) -- (1.2,0,0) node[pos=1.2]{$x$};
    \draw[->] (0,0,0) -- (0,1.1,0) node[pos=1.1]{$y$};
    \draw[->] (0,0,0) -- (0,0,2.0) node[pos=1.05]{$f(x,y)$};

    \def\n{10}                      %
    \pgfmathsetmacro{\step}{1/\n}   %

    \foreach \i in {1,...,\n}{
      \foreach \j in {1,...,\n}{
        \pgfmathsetmacro{\x}{(\i-1)*\step}
        \pgfmathsetmacro{\y}{(\j-1)*\step}
        \pgfmathsetmacro{\xp}{\i*\step}
        \pgfmathsetmacro{\yp}{\j*\step}
        \pgfmathsetmacro{\hp}{f(\xp,\yp)}

        \coordinate (A)  at ({\x},{\y},0);
        \coordinate (B)  at ({\xp},{\y},0);
        \coordinate (C)  at ({\xp},{\yp},0);
        \coordinate (D)  at ({\x},{\yp},0);

        \coordinate (A') at ({\x},{\y},{\hp});
        \coordinate (B') at ({\xp},{\y},{\hp});
        \coordinate (C') at ({\xp},{\yp},{\hp});
        \coordinate (D') at ({\x},{\yp},{\hp});

        \begin{scope}[fill=blue!30, fill opacity=0.5, draw=none]
          \fill (A) -- (D) -- (D') -- (A') -- cycle;
          \fill (C) -- (D) -- (D') -- (C') -- cycle;
          \fill (B) -- (C) -- (C') -- (B') -- cycle;
        \end{scope}

        \draw[blue, thick]
          (A)--(D)--(D')--(A') 
          (C)--(D)--(D')--(C') 
          (B)--(C)--(C')--(B')
          (B)--(B') (C)--(C') (D)--(D');

        \draw[blue, thick, fill=blue!70, fill opacity=0.8]
          (A')--(B')--(C')--(D')--cycle;
      }
    }

  \begin{scope}[fill=red!20, fill opacity=0.3, draw=none]
    \foreach \i in {1,...,\n}{
      \foreach \j in {1,...,\n}{
        \pgfmathsetmacro{\x}{(\i-1)*\step}
        \pgfmathsetmacro{\y}{(\j-1)*\step}
        \pgfmathsetmacro{\xp}{\i*\step}
        \pgfmathsetmacro{\yp}{\j*\step}
        \pgfmathsetmacro{\hA}{f(\x,\y)}
        \pgfmathsetmacro{\hB}{f(\xp,\y)}
        \pgfmathsetmacro{\hC}{f(\xp,\yp)}
        \pgfmathsetmacro{\hD}{f(\x,\yp)}
        \fill 
          ({\x},{\y},{\hA})
          -- ({\xp},{\y},{\hB})
          -- ({\xp},{\yp},{\hC})
          -- ({\x},{\yp},{\hD})
        -- cycle;
      }
    }
  \end{scope}

    \begin{scope}[draw=red, line width=1pt]
      \foreach \X in {0,0.25,...,1}
        \draw plot[variable=\y,domain=0:1,smooth] (\X,\y,{f(\X,\y)});
      \foreach \Y in {0,0.25,...,1}
        \draw plot[variable=\x,domain=0:1,smooth] (\x,\Y,{f(\x,\Y)});
    \end{scope}

  \end{tikzpicture}
   \caption{\label{fig:approximation}Illustration of Eq.~\eqref{eq:approximation-step}. In {\color{red}red} the function $f$ to be approximated; in {\color{blue}blue} the approximation $f_{10}$.}
\end{wrapfigure}

\begin{theorem}[Riesz–Markov in our setting~\citep{kakutani41}] 
\label{th:riesz}
Let $\mathcal{X} = [0,1]^d$ be the input space and $C^0(\mathcal{X})$ be the space of continuous function from $\mathcal{X}$ to $\Reals$. For any \texttt{Linear} and \texttt{FSC} feature attribution $\phi: \mathcal{X} \times C^0([0,1]^d) \to \Reals^d$, there are $d$ unique signed regular Borel measures $\{\mu_{j, \xbf}\}_{j \in \llbracket d \rrbracket}$ on $[0,1]^d$, with $\xbf \in [0,1]^d$, such that:
\[
\forall \xbf \in [0,1]^d, \forall f \in C^0([0,1]^d), \forall j \in \llbracket d \rrbracket,  \qquad  \phi(\xbf, f)_j = \int_{[0,1]^d} f(\ybf) \, \Diff \mu_{j, \xbf}(\ybf) \, . 
\]
\end{theorem}

In other words, \textbf{any \texttt{FSC} and \texttt{Linear} feature attribution method for models in $C^0([0,1]^d)$ is just an integral.} 
In the setting of Section~\ref{sec:3.1} with $d = 2$, the measures $\{\mu_{j, \xbf}\}_{j \in \llbracket 2 \rrbracket}$ from Th.~\ref{th:riesz} take the form
$\Diff \mu_{1, \xbf} (x, y) \defeq \Diff \delta_{\xbf_1}(x) \Diff y$ and $\Diff \mu_{2, \xbf} (x, y) \defeq \Diff x \Diff \delta_{\xbf_2}(y)$, thereby recovering the attribution of Equation~\eqref{eq:pdp-2d}.
Here, $\delta_{\xbf_1}$ denotes the Dirac measure, which concentrates all the weight at the point $\xbf_1$.
Or more succinctly, $\mu_{1, \xbf} \defeq \delta_{\xbf_1} \otimes \mathcal{L}$ and $\mu_{2, \xbf} \defeq \mathcal{L} \otimes \delta_{\xbf_2}$ with $\mathcal{L}$ the Lebesgue measure on $[0,1]$.

\begin{remark}
\label{rmk:fsc}
If the feature attribution $\phi$ is assumed to be positive, the \texttt{FSC} property is automatically satisfied and the representing measure $\mu_{j, \xbf}$ is positive. 
\end{remark}

\texttt{FSC} can seem very strong but can be interpreted as follows: \emph{two models which have $\epsilon>0$ close attribution (in term of $\norm{\cdot}_2$) should be $\delta(\epsilon)> 0$ distant to each other (in term of $\norm{\cdot}_\infty$)}. 
This is a weak version of the \texttt{Model Invariance} property~\citep{sundararajan2017axiomatic} (Appendix~\ref{app:2.1}). 
Another argument in favor of \texttt{FSC} is given in Remark~\ref{rmk:fsc}, as it is naturally checked by \texttt{Linear} and positive feature attribution methods.

\begin{remark} 
All the theorems in the following sections can be formulated with \texttt{FPC} (see Appendix~\ref{app:3.2}) instead of \texttt{FSC}. Moreover, the topology induced by \texttt{FPC} is weaker than the one generated by \texttt{FSC}.
One should note that Th.~\ref{th:riesz} does not apply with \texttt{FPC}.
For simplicity, we state all the results with \texttt{FSC}.
\end{remark}

\subsection{\texorpdfstring{General case: $d$-dimensional feature attribution}{General case: d-dimensional feature attribution}} 
\label{sec:3.3}

In this section, we extend the construction from Section~\ref{sec:3.1} to the $d$-dimensional input space. 
In doing so, we also specify the measures in Th.~\ref{th:riesz} by defining the attribution assigned to indicator functions $\indic{R}$, where $R \subset [0,1]^d$ is a hyperrectangle.
Our main result is the following:

\begin{theorem}[From atomic to general attribution]
\label{th:measure-atomic-attribution}
Let $\phi$ be a \texttt{Linear} and \texttt{FSC} attribution method.
Let $\{ \mu_{j, \xbf} \}_{j \in \llbracket d \rrbracket,\xbf\in \Inputspace}$
be a family of finite signed Borel \emph{measures} on $[0,1]^d$. 
Assume that the atomic attributions of indicator functions are given by:
\begin{align*}
\forall \xbf \in [0,1]^d, \forall j \in \llbracket d \rrbracket, \qquad
\phi\left(\xbf, \indic{R} \right)_j \defeq 
\mu_{j, \xbf}\left( R \right)  \in \Reals \, ,
\end{align*}
where $R \subset [0,1]^d$ is a hyperrectangle.
Then, the attribution of a continuous model $f: [0,1]^d \to \Reals$ is the following Lebesgue-Stieltjes integral:
\begin{align*}
\forall \xbf \in [0,1]^d, \forall j \in \llbracket d \rrbracket, \qquad \phi\left(\xbf, f \right)_j = \int_{[0,1]^d} f(\ybf) \, \Diff \mu_{j, \xbf}(\ybf)
\, .
\end{align*}
\end{theorem}

One should read Th.~\ref{th:measure-atomic-attribution} as follows: \textbf{first, decide on an attribution for the indicator function $\phi\left(\xbf, \indic{R} \right)$, then, get the attribution for the whole model $\phi\left(\xbf, f \right)$ in the form of an integral.}
The intuition behind the attribution $\phi\left(\xbf, \indic{R} \right)_j \defeq \mu_{j, \xbf}(R)$ is clear: it expresses the attribution of $\indic{R}$ in terms of the measure of the region $R$. 
This reflects how important feature $j$ is for activating $\indic{R}$, where the notion of importance is determined by the measure chosen for that feature.
In other words: \textbf{in our framework, doing feature attribution is equivalent to choosing a family of measures over the input space $\mathcal{X}$}. 

To simplify the selection of the measures, we can restrict our attention to measures $\mu$ that are absolutely continuous with respect to the Lebesgue measure on $[0,1]^d$, meaning they can be written as $\Diff \mu(\xbf) = h(\xbf) \Diff \xbf$ with $h : \Reals^d \to \Reals$ a density function. 
This restriction enhances intuition because the density $h$ can then be interpreted as a weight function over the input space $[0,1]^d$. 
This setting is presented in Appendix~\ref{app:3.3}.

In practice, the implementation of our framework is done in two manners: either Monte-Carlo with a sampler for high dimensional models or a classical Riemann sum approximation of the integral which is vectorized on GPU. 
It is compatible with any PyTorch model and allows the use of arbitrary integrands/measures to explore various instantiations of our feature attribution framework.

\section{Applications of our framework}
\label{sec:4}

Now that our feature attribution framework is defined, we illustrate its applications through several examples, which span through theory and practice.

\subsection{Feature attribution for piecewise affine functions} 
\label{sec:4.1}

We derive a closed-form representation of the feature attribution for the class of piecewise affine continuous functions~$\mathcal{F}$, defined as:

\begin{definition}[Piecewise affine continuous function]\label{def:piecewise-affine}
We say that a continuous function $f:~[0,1]^d \to \Reals$ is piecewise affine continuous if it can be written as
\[
\forall \xbf \in [0,1]^d, \qquad f(\xbf) \defeq \sum_{P \in \mathcal{R}([0,1]^d)} \left(\abf_{P}^\top \xbf + b_P \right) \indic{\xbf \in P} 
\, ,
\]
where $\mathcal{R}([0,1]^d)$ is a partition of $[0,1]^d$ into $d$-polytopes (\emph{i.e.}, any shape in $[0,1]^d$ formed by gluing together finitely many convex $d$-dimensional pieces along their faces), and $\abf_{P} \in \Reals^d$ and $b_P \in \Reals$ are the coefficients of the local linear model in region $P$. 
\end{definition}

Interestingly, \textbf{this class of functions corresponds exactly to the deep ReLU networks}~\citep{petersen2024mathematical} as defined in Appendix~\ref{app:4.1}. 
The following result, direct consequence of Th.~\ref{th:measure-atomic-attribution}, provides a representation of feature attributions for piecewise affine continuous functions, making it directly applicable to deep ReLU networks.

\begin{corollary}[Feature attribution of piecewise affine continuous functions] \label{cor:linear-positive}
Let $\phi$ be a \texttt{Linear} and \texttt{FSC} attribution method.
Let $\{ \mu_{j, \xbf} \}_{j \in \llbracket d \rrbracket,\xbf\in \Inputspace}$
be a family of finite positive Borel \emph{measures} on $[0,1]^d$. 
Assume that the atomic attributions of indicator functions are given by:
\begin{align*}
\forall \xbf \in [0,1]^d, \forall j \in \llbracket d \rrbracket, \qquad
\phi\left(\xbf, \indic{R} \right)_j \defeq 
\mu_{j, \xbf} (R)  \in \Reals_+ \, ,
\end{align*}
where $R \subset [0,1]^d$ is a hyperrectangle.
Then, the attribution of a piecewise affine continuous function~$f$ is given by: 
\begin{align*}
\forall \xbf \in [0,1]^d, \forall j \in \llbracket d \rrbracket, \qquad \phi\left(\xbf, f \right)_j = \sum_{P \in \mathcal{R}([0,1]^d)} \mu_{j, \xbf} (P) \left( \abf_{P}^\top \mathbf{m}^{(P, j, \xbf)} + b_P \right)  \, ,
\end{align*}
where $\mathbf{m}^{(P, j, \xbf)} \in \Reals^d$ is the \emph{center of mass} of $P$ defined as:
\[
\forall i \in \llbracket d \rrbracket, \qquad \mathbf{m}^{(P, j, \xbf)}_i \defeq 
\begin{cases}
(\mu_{j, \xbf}(P))^{-1} \int_{P} \ybf_i \, \Diff \mu_{j, \xbf} (\ybf)  & \text{if $\mu_{j, \xbf}(P) \neq 0$} \, ,\\
0 & \text{otherwise} \, .
\end{cases} 
\]
\end{corollary}	
Informally, for this class of models, feature attribution is computed as a weighted sum of the model's output evaluated at the center of mass of each $P \in \mathcal{R}([0,1]^d)$. 
The sum is taken over all elements of the input space partition $\mathcal{R}([0,1]^d)$, with the weight of each element given by its corresponding explanation measure $\mu(P)$.
Notably, \textbf{this corollary also extends to CART trees and random forests} by relying only on \texttt{Linearity} and the atomic attribution of indicator functions.
\texttt{FSC} is not needed as CART-trees create a \emph{fixed} partition of $\mathcal{X}$ with axis-aligned hyperrectangle, and the sole purpose of \texttt{FSC} is to take the limit of the partition.

\begin{remark}
For ReLU networks, the coefficients $(\abf_P, b_P)$ of the linear regions $P$ have closed-form expressions in terms of the network's weights. 
See Appendix~\ref{app:4.1} for details. 
\end{remark}

\subsection{Recovering existing feature attribution methods in our framework}
\label{sec:4.2}

\begin{table}[t]
\caption{\label{tab:example-feature-attributions}Specialization of the measures from Th.~\ref{th:measure-atomic-attribution} to recover existing feature attribution methods. 
Recall that $\xbf_{-j} \in \Reals^{d-1}$ is $\xbf$ without its $j$-th coordinate.
}
\centering
\begin{tabular}{llr}
\toprule
\textbf{Atomic attribution $\phi\left(\xbf, \indic{R} \right)_j$} & \textbf{Attribution for the model $\phi\left(\xbf, f \right)_j$} & \textbf{Method}  \\
\midrule
$\mathbb{P}_{\Xbf}(\Diff \ybf\mid \Xbf_j = \xbf_j)$ & $\condexpec{f(\Xbf)}{\Xbf_j = \xbf_j}$ & \citep{covert2021explaining} \\
$\prod_{i = 1}^{j-1} \mathbb{P}_{\Xbf_i}(\Diff \ybf_i)  \delta_{\xbf_j}(\Diff \ybf_j) \prod_{k = j+1}^{d} \mathbb{P}_{\Xbf_k}(\Diff \ybf_k) $  & $\expecunder{f\left(\Xbf^{(\xbf_j, j)}\right)}{\bigotimes_{i \in \llbracket d \rrbracket, i\neq j} \mathbb{P}_{\Xbf_i}}$ & \citep{covert2021explaining} \\
$ \delta_{\xbf_j}(\Diff  \ybf_j) \mathbb{P}_{\Xbf_{-j}}(\Diff  \ybf_{-j})$ & $\expecunder{f\left(\Xbf^{(\xbf_j, j)}\right)}{\mathbb{P}_{\Xbf_{-j}}}$ & \citep{friedman2001greedy} \\
\bottomrule
\end{tabular}
\end{table}

In Table~\ref{tab:example-feature-attributions}, we show how particular choices of measures leads to recovering existing feature attribution methods. 
All the methods listed in Table~\ref{tab:example-feature-attributions} can be expressed as expectations, making it possible to recover them within our framework. 
The first entry frames feature attribution in the feature removal setting of~\citep{covert2021explaining}, specifically, by marginalizing features with respect to a conditional distribution.
The second entry assumes feature independence and performs marginalization using the product of marginals $\mathbb{P}_{\Xbf_i}$. 
The third entry corresponds exactly to PDP where the explained feature is fixed and the remaining features are marginalized out using the remaining joint distribution.
These approaches all share a common principle: feature attribution is defined with respect to some assumed data-generating distribution, typically approximated by the empirical dataset. 

We now turn to linear models. 
A natural feature attribution for $f_\wbf: \xbf \mapsto \wbf^\top \xbf$ is to look at the coefficients $\wbf_j$. 
Using the family of measures
\[
\mu_{j}(\Diff  \ybf)  \defeq  2 \prod_{i=1}^{j-1} \delta_0(\Diff  \ybf_i)  \, \Diff \ybf_j \prod_{k=j+1}^{d} \delta_0(\Diff \ybf_k)  \, ,
\]
we can recover the coefficients $\wbf_j$. 
This fact is made rigorous in Appendix~\ref{app:4.2}.

\subsection{Optimizing the feature attribution}\label{sec:4.3}

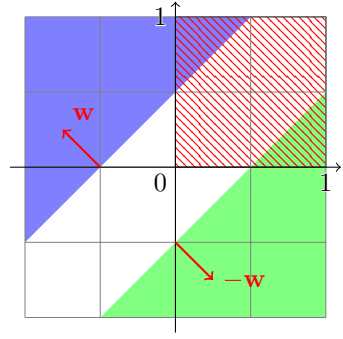
\begin{wrapfigure}{!t}{0.4\textwidth}
  \centering
      \begin{tikzpicture}
    
        \fill[green!50] (-1,-2) -- (2,-2) -- (2,1);
        \fill[blue!50] (-2,-1) -- (-2,2) -- (1,2);
        \draw[pattern=north west lines, pattern color=red] (0,0) rectangle (2,2);
        
        \draw[very thin,color=gray] (-2,-2) grid (2,2);
        \draw[->] (-2.2,0) -- (2.2,0) node[right] {};
        \draw[->] (0,-2.2) -- (0,2.2) node[above] {};
        \node at (2,-0.2) {\contour{white}{$1$}};
        \node at (-0.2,2) {\contour{white}{$1$}};
        \node at (-0.2,-0.2) {$0$};

         \draw[thick, ->, red] (-1,0) -- (-1.5,1/2) node[red, above right] {$\wbf$};
         \draw[thick, ->, red] (0,-1) -- (0.5,-1.5) node[red, right] {$-\wbf$};;

    \end{tikzpicture}
\caption{\label{fig:regression-optimal-measures-1}Description of the solution to \eqref{eq:pb-optim-2} for $f_\wbf$ with $d = 2$ and $\wbf \in \Reals^2$. The {\color{green}green} and {\color{blue}blue} areas correspond to the solution set $\mathcal{S}_j$ ($j \in \llbracket 2 \rrbracket$) of Th.~\ref{th:optim-1} when $j \in D_{1, \wbf}$. The {\color{red}red} hatched area represent the input space. 
All measures such that their center of mass belong to the {\color{green}green} or {\color{blue}blue} areas are optimal for \eqref{pb:recall-1}.
}
\end{wrapfigure}
Our framework is natural for optimization, as it provides a well-defined search space: the space of measures. 
However, optimization in measure space is generally difficult and cannot be performed in closed form~\citep{molchanov2000tangent}. 
Nevertheless, for piecewise affine models, this optimization problem reduces to a finite-dimensional problem. 
We detail this procedure when $\phi$ is \emph{global} and the criterion for the quality of attribution is \emph{Recall}~\citep{ribeiro2016should}. 
For linear models with a ground-truth feature attribution available, it is defined as:

\begin{definition}[Recall]  
\label{def:formal-recall}
Given a threshold $\beta > 0$ and a linear model $f_\wbf : [0,1]^d \to \Reals$, with weights $\wbf \in \Reals^d$, and the index partition $D_{1, \wbf} \sqcup D_{0, \wbf} = \llbracket d \rrbracket$, where $D_{0, \wbf} \defeq \{ i \in \llbracket d \rrbracket : \abs{\wbf_i} \leq \beta \}$ and $D_{1, \wbf} \defeq D_{0, \wbf}^\complement$ (called the \emph{golden set}). 
    Also, let $\phi$ be a global feature attribution method which depends on a vector of measures $\mu \defeq ( \mu_j )_{j \in \llbracket d \rrbracket} \in \mathcal{P}^d$. 
	We define the \emph{Recall} metric, with a threshold $\alpha > 0$, as
	\[
	 \mathrm{Recall}_{\alpha, \wbf}(\mu) \defeq \frac{\sum_{j \in D_{1, \wbf}} \indic{\abs{\phi(f_\wbf)_j} \geq \alpha}}{\sum_{j \in D_{1, \wbf}} \indic{\abs{\phi(f_\wbf)_j} \geq \alpha} + \sum_{j \in D_{1, \wbf}} \indic{\abs{\phi(f_\wbf)_j} < \alpha}} = \frac{\sum_{j \in D_{1, \wbf}} \indic{\abs{\phi(f_\wbf)_j} \geq \alpha}}{\card{D_{1, \wbf}}}  \, ,
\]
where $\mathcal{P}^d$ is the space of probability measures on $[0,1]^d$.
\end{definition}

Informally, \emph{Recall} measures the proportion of features retrieved by $\phi$ that are actually used by the model $f_\wbf$. 
Thus, ideally, we would like to solve 
\begin{equation}
\label{pb:recall-1}
\Argmax_{\mu \in \mathcal{P}^d} \mathrm{Recall}_{\alpha, \wbf}(\mu) 
\, .    
\end{equation}
This can be casted into an optimization problem in $[0,1]^d$ (see Appendix~\ref{app:4.3}) as 
\begin{equation}\label{eq:pb-optim-2}
 \Argmax_{\mbf^{(1)},  \ldots, \mbf^{(d)} \in  [0,1]^d } \mathrm{Recall}_{\alpha, \wbf}(\mbf^{(1)}, \ldots, \mbf^{(d)}) \, .
\end{equation}
The measures solving Problem~(\ref{pb:recall-1}) are generally model-specific, \emph{i.e.}, they depend on the particular choice of $\wbf$ and are denoted by $\mu_{\wbf}$.
To identify measures that perform well across a class of models, we can intersect the sets of optimal solutions obtained for individual models, provided the intersection is non-empty.
That is why, defining the evaluation metric on specific models is natural, since finding a measure that is optimal for all possible unconstrained models is, in general, not feasible. 
Using geometric arguments, one can show that the measures $\mu$ which optimize $\mathrm{Recall}_{\alpha, \wbf}$ are projected on vectors that lie above a given hyperplane as follows
\begin{theorem}[Optimal projections for $\mathrm{Recall}_{\alpha,\mathbf{w}}$] \label{th:optim-1}
For each $j \in \llbracket d \rrbracket$ define
	\begin{align*}
		\mathcal{S}_j \defeq 
		\begin{cases}
			\left\{\mbf \in [0,1]^d : \wbf^\top \mbf \geq \alpha \right\} \cup \left\{\mbf \in [0,1]^d : \wbf^\top \mbf \leq -\alpha \right\} \text{  if $j \in D_{1, \wbf}$} \, ,\\
			[0,1]^d  \text{  otherwise.}
		\end{cases}
	\end{align*}
Then, the solutions of \eqref{eq:pb-optim-2} are $\mbf^{(j)}\in \mathcal{S}_j$ for all $j\in \llbracket d\rrbracket$. 
\end{theorem}

This theorem is illustrated in Figure~\ref{fig:regression-optimal-measures-1}. 
It should be interpreted as a sanity check for the optimality of a given measure and can thus aid in designing good measures. 
More precisely, to check if a measure $\mu$ is optimal, we compute its center of mass (see Appendix~\ref{app:4.3}) and see if it lies inside of a given solution set $\mathcal{S}_j$.
Also, it is possible to construct trivial or degenerate measures that project to a given vector, but these are typically model-dependent. 

There are two main limitations to this result.
First, it relies on the availability of a ground-truth for the features, which is not always the case. 
Second, it does not provide a direct characterization of the optima in the space of measures.
Instead, it characterizes the optima in $[0,1]^d$ after projecting the measures. 
Since this projection is not bijective, it seems that there is no straightforward way to recover all possible measures that map to a given vector in $[0,1]^d$. 
Additional analysis related to the \emph{Precision} metric, a natural counterpart to \emph{Recall}, and the optimization of feature attribution for ReLU networks can be found in Appendix~\ref{app:4.3}.

\section{Conclusion}

In this paper, we showed that any attribution method satisfying the axioms introduced in~\citep{sundararajan2017axiomatic} can produce attributions similar to \emph{Gradient$\times$Input}. 
Based on this observation, we propose a new framework for feature attribution. 
By defining attributions for indicator functions and making mild assumptions on the feature attribution method, one can reconstruct the attribution for the entire model. 
We then express several existing XAI methods within our framework and derive closed-form attributions for deep ReLU networks. 
An implementation of this framework is available online.
Finally, we use our framework to introduce a sanity check for feature attribution methods and take a first step toward theoretically optimizing feature attribution with respect to a given evaluation metric.

As future work, we want to explore the optimization of feature attribution, as introduced in Section~\ref{sec:4.3}, along two main directions. 
First, we aim to better characterize the properties of attributions that optimize a given evaluation metric, and, if possible, derive closed-form solutions. 
Second, on the empirical side, we plan to apply the gradient ascent algorithm of \citep{molchanov2002steepest} in the context of feature attribution, providing a way of empirically constructing optimal feature attribution methods.

\medskip

\bibliography{biblio}

\begin{thebibliography}{10}

\bibitem{adebayo2018sanity}
Julius Adebayo, Justin Gilmer, Michael Muelly, Ian Goodfellow, Moritz Hardt,
  and Been Kim.
\newblock {Sanity Checks for Saliency Maps}.
\newblock {\em Advances in Neural Information Processing Systems}, 2018.

\bibitem{agarwal2022openxai}
Chirag Agarwal, Satyapriya Krishna, Eshika Saxena, Martin Pawelczyk, Nari
  Johnson, Isha Puri, Marinka Zitnik, and Himabindu Lakkaraju.
\newblock {OpenXAI: Towards a Transparent Evaluation of Model Explanations}.
\newblock {\em Advances in Neural Information Processing Systems}, 2022.

\bibitem{aistleitner2014functions}
Christoph Aistleitner and Josef Dick.
\newblock {Functions of bounded variation, signed measures, and a general
  Koksma-Hlawka inequality}.
\newblock {\em Acta Arithmetica}, 2015.

\bibitem{ancona2017towards}
Marco Ancona, Enea Ceolini, Cengiz Öztireli, and Markus Gross.
\newblock {Towards better understanding of gradient-based attribution methods
  for Deep Neural Networks}.
\newblock In {\em International Conference on Learning Representations}, 2018.

\bibitem{bach2015pixel}
Sebastian Bach, Alexander Binder, Gr{\'e}goire Montavon, Frederick Klauschen,
  Klaus-Robert M{\"u}ller, and Wojciech Samek.
\newblock {On Pixel-Wise Explanations for Non-Linear Classifier Decisions by
  Layer-Wise Relevance Propagation}.
\newblock {\em PloS one}, 2015.

\bibitem{benitez_et_al_1997}
Jos{\'e}~Manuel Ben{\'\i}tez, Juan~Luis Castro, and Ignacio Requena.
\newblock {Are Artificial Neural Networks Black Boxes?}
\newblock {\em IEEE Transactions on Neural Networks}, 1997.

\bibitem{bhattacharjee2025safely}
Robi Bhattacharjee, Karolin Frohnapfel, and Ulrike von Luxburg.
\newblock {How to safely discard features based on aggregate SHAP values}.
\newblock {\em arXiv preprint arXiv:2503.23111}, 2025.

\bibitem{bilodeau2024impossibility}
Blair Bilodeau, Natasha Jaques, Pang~Wei Koh, and Been Kim.
\newblock {Impossibility Theorems for Feature Attribution}.
\newblock {\em Proceedings of the National Academy of Sciences}, 2024.

\bibitem{Breneis2020}
Simon Breneis.
\newblock Functions of bounded variation in one and multiple dimensions.
\newblock Master's thesis, Johannes Kepler University Linz, 2020.

\bibitem{bressan2024theory}
Marco Bressan, Nicol{\`o} Cesa-Bianchi, Emmanuel Esposito, Yishay Mansour, Shay
  Moran, and Maximilian Thiessen.
\newblock {A Theory of Interpretable Approximations}.
\newblock In {\em The Thirty Seventh Annual Conference on Learning Theory}.
  PMLR, 2024.

\bibitem{courant2000introduction}
Richard Courant and John Fritz.
\newblock {\em {Introduction to Calculus and Analysis (volume II)}}.
\newblock Classics in Mathematics. Springer-Verlag Berlin Heidelberg, 2000.

\bibitem{covert2021explaining}
Ian Covert, Scott Lundberg, and Su-In Lee.
\newblock {Explaining by Removing: A Unified Framework for Model Explanation}.
\newblock {\em Journal of Machine Learning Research}, 2021.

\bibitem{fokkema2023attribution}
Hidde Fokkema, Rianne De~Heide, and Tim Van~Erven.
\newblock {Attribution-based Explanations that Provide Recourse Cannot be
  Robust}.
\newblock {\em Journal of Machine Learning Research}, 2023.

\bibitem{folland1999real}
Gerald~B Folland.
\newblock {\em Real analysis: modern techniques and their applications}.
\newblock John Wiley \& Sons, 1999.

\bibitem{friedman2001greedy}
Jerome~H Friedman.
\newblock {Greedy Function Approximation: A Gradient Boosting Machine}.
\newblock {\em Annals of Statistics}, 2001.

\bibitem{garreau2021does}
Damien Garreau and Dina Mardaoui.
\newblock {What does LIME really see in images?}
\newblock In {\em International Conference on Machine Learning}. PMLR, 2021.

\bibitem{ghorbani2019interpretation}
Amirata Ghorbani, Abubakar Abid, and James Zou.
\newblock {Interpretation of Neural Networks is Fragile}.
\newblock In {\em Proceedings of the AAAI Conference on Artificial
  Intelligence}, 2019.

\bibitem{guptaetal2022}
Arushi Gupta, Nikunj Saunshi, Dingli Yu, Kaifeng Lyu, and Sanjeev Arora.
\newblock {New Definitions and Evaluations for Saliency Methods: Staying
  Intrinsic, Complete and Sound}.
\newblock In {\em Advances in Neural Information Processing Systems}, 2022.

\bibitem{han2022explanation}
Tessa Han, Suraj Srinivas, and Himabindu Lakkaraju.
\newblock {Which Explanation Should I Choose? A Function Approximation
  Perspective to Characterizing Post Hoc Explanations}.
\newblock {\em Advances in Neural Information Processing Systems}, 2022.

\bibitem{hedstrom2023quantus}
Anna Hedstr{\"o}m, Leander Weber, Daniel Krakowczyk, Dilyara Bareeva, Franz
  Motzkus, Wojciech Samek, Sebastian Lapuschkin, and Marina M-C H{\"o}hne.
\newblock {Quantus: An Explainable AI Toolkit for Responsible Evaluation of
  Neural Network Explanations and Beyond}.
\newblock {\em Journal of Machine Learning Research}, 2023.

\bibitem{heo2019fooling}
Juyeon Heo, Sunghwan Joo, and Taesup Moon.
\newblock {Fooling Neural Network Interpretations via Adversarial Model
  Manipulation}.
\newblock {\em Advances in Neural Information Processing Systems}, 2019.

\bibitem{hesse2021fast}
Robin Hesse, Simone Schaub-Meyer, and Stefan Roth.
\newblock {Fast Axiomatic Attribution for Neural Networks}.
\newblock {\em Advances in Neural Information Processing Systems}, 2021.

\bibitem{hildebrandt1963introduction}
T.~H. Hildebrandt.
\newblock {\em {Introduction to the Theory of Integration}}.
\newblock Pure and Applied Mathematics. Academic Press, 1963.

\bibitem{hooker2019benchmark}
Sara Hooker, Dumitru Erhan, Pieter-Jan Kindermans, and Been Kim.
\newblock {A Benchmark for Interpretability Methods in Deep Neural Networks}.
\newblock {\em Advances in Neural Information Processing Systems}, 2019.

\bibitem{humayun2023splinecam}
Ahmed~Imtiaz Humayun, Randall Balestriero, Guha Balakrishnan, and Richard~G
  Baraniuk.
\newblock {SplineCam: Exact Visualization and Characterization of Deep Network
  Geometry and Decision Boundaries}.
\newblock In {\em Proceedings of the IEEE/CVF Conference on Computer Vision and
  Pattern Recognition}, 2023.

\bibitem{kakutani41}
Shizuo Kakutani.
\newblock {Concrete Representation of Abstract (M)-Spaces (A characterization
  of the Space of Continuous Functions)}.
\newblock {\em Annals of Mathematics}, 1941.

\bibitem{kindermans2019reliability}
Pieter-Jan Kindermans, Sara Hooker, Julius Adebayo, Maximilian Alber, Kristof~T
  Sch{\"u}tt, Sven D{\"a}hne, Dumitru Erhan, and Been Kim.
\newblock {The (Un)reliability of saliency methods}.
\newblock {\em Explainable AI: Interpreting, explaining and visualizing deep
  learning}, 2019.

\bibitem{konig2024disentangling}
Gunnar K{\"o}nig, Eric G{\"u}nther, and Ulrike von Luxburg.
\newblock {Disentangling Interactions and Dependencies in Feature Attribution}.
\newblock {\em arXiv preprint arXiv:2410.23772}, 2024.

\bibitem{lopardo2024attention}
Gianluigi Lopardo, Frédéric Precioso, and Damien Garreau.
\newblock {Attention Meets Post-hoc Interpretability: A Mathematical
  Perspective}.
\newblock In {\em nternational Conference on Machine Learning}, 2024.

\bibitem{lundberg2017unified}
Scott~M Lundberg and Su-In Lee.
\newblock {A Unified Approach to Interpreting Model Predictions}.
\newblock {\em Advances in Neural Information Processing Systems}, 2017.

\bibitem{Magyar1992Continuous}
Zolt{\'a}n Magyar.
\newblock {\em {Continuous Linear Representations}}.
\newblock North-Holland Mathematics Studies. Elsevier, 1992.

\bibitem{molchanov2000tangent}
Ilya Molchanov and Sergei Zuyev.
\newblock {Tangent Sets in the Space of Measures: With Applications to
  Variational Analysis}.
\newblock {\em Journal of mathematical analysis and applications}, 2000.

\bibitem{molchanov2002steepest}
Ilya Molchanov and Sergei Zuyev.
\newblock Steepest descent algorithms in a space of measures.
\newblock {\em Statistics and Computing}, 2002.

\bibitem{owen2005multidimensional}
Art~B Owen.
\newblock {Multidimensional Variation for Quasi-Monte Carlo}.
\newblock In {\em Contemporary Multivariate Analysis And Design Of Experiments:
  In Celebration of Professor Kai-Tai Fang's 65th Birthday}. World Scientific,
  2005.

\bibitem{petersen2024mathematical}
Philipp Petersen and Jakob Zech.
\newblock Mathematical theory of deep learning.
\newblock {\em arXiv preprint arXiv:2407.18384}, 2024.

\bibitem{ribeiro2016should}
Marco~Tulio Ribeiro, Sameer Singh, and Carlos Guestrin.
\newblock {"Why Should I Trust You?": Explaining the Predictions of Any
  Classifier}.
\newblock In {\em Proceedings of the 22nd ACM SIGKDD International Conference
  on Knowledge Discovery and Data Mining}, 2016.

\bibitem{rong2022consistent}
Yao Rong, Tobias Leemann, Vadim Borisov, Gjergji Kasneci, and Enkelejda
  Kasneci.
\newblock {A Consistent and Efficient Evaluation Strategy for Attribution
  Methods}.
\newblock {\em International Conference on Machine Learning}, 2022.

\bibitem{rudin1976principles}
Walter Rudin.
\newblock {\em {Principles of Mathematical Analysis}}.
\newblock McGraw-Hill, New York, 1976.

\bibitem{selvaraju2017grad}
Ramprasaath~R Selvaraju, Michael Cogswell, Abhishek Das, Ramakrishna Vedantam,
  Devi Parikh, and Dhruv Batra.
\newblock {Grad-CAM: Visual Explanations from Deep Networks via Gradient-based
  Localization}.
\newblock In {\em Proceedings of the IEEE international conference on computer
  vision}, 2017.

\bibitem{shrikumar2016not}
Avanti Shrikumar, Peyton Greenside, and Anshul Kundaje.
\newblock {Learning Important Features Through Propagating Activation
  Differences}.
\newblock In {\em International Conference on Machine Learning}. PMLR, 2017.

\bibitem{sundararajan2017axiomatic}
Mukund Sundararajan, Ankur Taly, and Qiqi Yan.
\newblock {Axiomatic Attribution for Deep Networks}.
\newblock In {\em International Conference on Machine Learning}. PMLR, 2017.

\bibitem{verdinelli2024feature}
Isabella Verdinelli and Larry Wasserman.
\newblock {Feature Importance: A Closer Look at Shapley Values and LOCO}.
\newblock {\em Statistical Science}, 2024.

\end{thebibliography}
\bibliographystyle{plain}

\clearpage
\appendix

\section*{Appendix}
\addcontentsline{toc}{section}{Appendix}
\printappendixtoc

\section{Appendix for Section~\ref{sec:limitations}} \label{app:2}

\subsection{Appendix for Section~\ref{sec:limitations.1}} \label{app:2.1}

Below, we present several axioms/properties discussed in the literature. The list is intentionally selective rather than exhaustive: we include only those that are either directly cited in our main results or for which we provide additional analysis.
Before introducing the first property, we need an auxiliary definition:
\begin{definition}[Action of a partition on vector] \label{def:action-vector}
	Given $S$ a partition of $\llbracket d \rrbracket$ and a vector $\Img \in  \Reals^d$, we define $\Img^S  \in \Reals^d$ as the unique vector verifying:
	\[
	\forall i \in \llbracket d \rrbracket , \exists! \, S' \in S , \qquad \left( i \in S' \implies \Img_i^{S} = \Img_{\min{S'}}   \right) \, .
	\]
\end{definition}
Informally, $\Img^S$ is obtained by making all components of $\Img$ that belong to the same subset $S' \in S$ equal to the component of $\Img$ at the smallest index in $S'$. Meaning, we assign the same value to all indices within each partition element.
We illustrate this definition as follows: given $d = 3$, the partition $S = \{ \{1,3\}, \{2\}\}$ and input $\Img \defeq (5, -2, 1)^\top$, then $\Img^{S} = (5,-2,5)^\top$.
\begin{definition}[Weak symmetry~\cite{sundararajan2017axiomatic}]
	A feature attribution method $\phi$ is \texttt{weakly-symmetric} if:
	\begin{align*}
	& \forall f \in \mathcal{F}, \forall \sigma \in \mathfrak{S}_d, \qquad \left(  \forall \xbf \in \Reals^d, \, f(\xbf) = f(g_\sigma(\xbf)) \implies  \forall \Img \in \Reals^d, \, \phi(\Img, f) = g_\sigma (\phi(\Img^{\llbracket d \rrbracket \small/ \sigma}, f)) \right) \, ,
	\end{align*}
	where $\Img^{\llbracket d \rrbracket \small/ \sigma}$ is defined as in Definition~\ref{def:action-vector}, using the partition $\llbracket d \rrbracket \small/ \sigma$  corresponding to the orbits of $\llbracket d \rrbracket$ under the action of $\sigma$ (\emph{i.e.}, two elements $x, y \in \llbracket d \rrbracket$ are in the same orbit if there exists $n \in \mathbb{N}$ s.t. $\sigma^n(x) = y$).
\end{definition}

For example, if we take the permutation $\sigma = (1 2) (3 5)$ (with $d = 5$), then $\llbracket d \rrbracket \big/ \sigma = \left\{  \{1,2 \}, \{3,5 \}, \{4\} \right\}$.
We introduces a stronger version of \texttt{Weak symmetry}, which avoids dependence on the partition $\llbracket d \rrbracket \small/ \sigma$ and leads to a more natural mathematical formulation.

\begin{definition}[Strong symmetry]
A feature attribution method $\phi$ is \texttt{strongly-symmetric} if:
	\begin{align*}
		& \forall f \in \mathcal{F}, \forall \sigma \in \mathfrak{S}_d, \qquad \left(  \forall \xbf \in \Reals^d, \, f(\xbf) = f(g_\sigma(\xbf)) \implies  \forall \Img \in \Reals^d, \, \phi(\Img, f) = g_\sigma (\phi(\Img, f)) \right) \, ,
	\end{align*}
	where $g_\sigma (\xbf) \defeq \left(  \xbf_{\sigma(1)}, \xbf_{\sigma(2)}, \ldots, \xbf_{\sigma(d)} \right)$ for $\xbf \in \Reals^d$ and $\mathfrak{S}_d$ is the permutation group of $\llbracket d \rrbracket$.
\end{definition}

To illustrate the difference between \texttt{Weak symmetry} and \texttt{Strong symmetry}, consider the simple model $f : \Reals^2 \to \Reals$ defined by $f(x, y) \defeq x + y$, and the permutation $\sigma = (1 \, 2)$ that swaps the two input features. 
Since $f(x, y) = f(y, x)$ for all $(x, y) \in \Reals^2$, the function $f$ is invariant under~$\sigma$.
A feature attribution method that satisfies \texttt{Weak symmetry} is only required to produce equal attributions when the input itself is invariant under $\sigma$. 
In other words:
\begin{equation} \label{eq:sym-weak}
    \forall x \in \Reals, \qquad \phi((x, x), f)_1 = \phi((x, x), f)_2 \, .
\end{equation}
By contrast, a \texttt{Strongly-symmetric} attribution method must produce equal attributions for all inputs whenever the function is invariant under $\sigma$, regardless of the specific input values:
\begin{equation} \label{eq:sym-strong}
    \forall (x, y) \in \Reals^2, \qquad \phi((x, y), f)_1 = \phi((x, y), f)_2 \, .
\end{equation}
This stricter requirement ensures that the symmetry of the function is fully reflected in the attributions across all inputs.
\begin{definition}[Model invariance~\cite{sundararajan2017axiomatic}]
	A feature attribution method $\phi$ is \texttt{model invariant} if:
	\begin{align*}
		& \forall f_0, f_1 \in \mathcal{F}, \qquad \left( \forall \xbf \in \Reals^d, \, f_0(\xbf) = f_1(\xbf) \implies \forall \Img \in \Reals^d, \, \phi(\Img,f_0 ) = \phi(\Img,f_1 ) \right) \, .
	\end{align*}
\end{definition}
In other words, if two models yield the same outputs for the same inputs, their attributions should be identical as well.
\begin{definition}[Robustness~\cite{fokkema2023attribution}]
	A feature attribution method $\phi$ is \texttt{robust} if:
	\begin{align*}
		\forall f \in \mathcal{F},  \qquad \phi(\cdot, f) \text{ is continuous} \, .
	\end{align*}
\end{definition}
\texttt{Robustness} ensures that the attributions are stable to the perturbations of the explained input $\xbf$.

\begin{definition}[Agreeing with linear models~\cite{verdinelli2024feature}]
A \emph{global} feature attribution method $\phi$ is agreeing with linear models if:
\[
\forall \wbf \in \Reals^d, \forall j \in \Reals^d, \qquad \phi(\wbf^\top \times (\cdot))_j = \wbf_j^2 \, .
\]
\end{definition}
In other words, this axiom treats linear models as a sanity check, since they provide a ground-truth attribution by design.

\subsection{Appendix for Section~\ref{sec:limitations.2}} \label{app:2.2}
In  this section, we prove the claims made at the end of Section~\ref{sec:limitations.2}.
First, one can not remove the baseline in the \texttt{Completeness} property as it creates a contradiction.

\begin{proposition}[Impossibility of completeness without baseline]
	Any \texttt{sensitive} and \texttt{complete} (without baseline) attribution method $\phi$ lead to a contradiction, where \texttt{Completeness} without baseline is defined as
	\begin{align*}
		\forall f \in \mathcal{F}, \forall \xbf \in \Reals^d, \qquad f(\xbf)  = \sum_{j \in \llbracket d \rrbracket} \phi(\xbf, f)_j  \, .
	\end{align*}
\end{proposition}

\begin{proof}
	Given a constant $c \in \Reals^\star$ and input $\xbf \in \Reals^d$, we define the constant function $\bar c : \ybf \mapsto c$. 
        By \texttt{Sensitivity} property, we get $\phi(\xbf, \bar c)_j = 0$, for all $j \in \llbracket d \rrbracket$.
	Now, using \texttt{Completeness} without baseline property
	\begin{align*}
		\bar c(\xbf) & = \sum_{j \in \llbracket d \rrbracket} \phi(\xbf, \bar c)_j \, .
	\end{align*}
    This implies that $c = 0$, which is absurd because $c \neq 0$.
\end{proof}
	Another rigidity of the given axioms/properties is that one can hardly make them local. For instance, trying the following modification of \texttt{Sensitivity} which seems quite natural introduces an incoherency.
\begin{proposition}[Absurdity of local sensibility]
	Given a baseline $\xbf' \in \Reals^d$ and $\delta > 0$, any attribution method $\phi$ which is  $\delta$-\texttt{sensitive}, and $\xbf'$-\texttt{complete} leads to a contradiction.
	We define (local) $\delta$-\texttt{sensitivity} as:
	\begin{align*}
		& \forall f \in \mathcal{F}, \forall \xbf_0 \in \Reals^d, \forall j \in \llbracket d \rrbracket, \quad \\
		&\left(  \forall \xbf \in B(\xbf_0, \delta), \forall x' \in [(\xbf_0)_j-\delta, (\xbf_0)_j+\delta] , \, f(\xbf) = f(\xbf^{(x',j)}) \implies  \forall \Img \in B(\xbf_0, \delta), \, \phi(\Img, f)_j = 0 \right) \, .
	\end{align*}
\end{proposition}
\begin{proof}
	We set the input dimension $d = 2$ and take the model $f(\xbf) \defeq \indic{\xbf_1 \geq 0}$. Also, take the baseline $\xbf' \in (\Reals_{-}^\star)^2$ such that $f(\xbf') = 0$.
	Now, by applying $\delta$-\texttt{sensitivity}, we get
	\begin{align*}
		\forall \xbf \in B(0, \delta)^\complement, \qquad \phi(\xbf, f) = 0.
	\end{align*}
	Finally, using \texttt{Completeness}, we have
	\begin{align*}
		\forall \xbf \in (\delta, +\infty)^2, \qquad  f(\xbf)  &= \sum_{j \in \llbracket d \rrbracket} \phi(\xbf, f)_j \, . 
	\end{align*}
    This implies that
    \[
    \forall \xbf \in (\delta, +\infty)^2, \qquad \indic{\xbf_1 \geq 0}  = 0 \, .
    \]
	But, $\indic{\xbf_1 \geq 0} \neq 0$ for $\xbf_1 \in (\delta, +\infty)$.
\end{proof}
Finally, another straightforward incoherency arises from the \texttt{Strong-symmetry} property, which we apply on the indicator functions
\begin{proposition}[Absurdity of strong-symmetry]
	If the model class $\mathcal{F}$ includes the indicator functions with respect to one coordinate. Then \texttt{Sensitivity}, \texttt{Completeness}  and \texttt{Linearity} are incompatible with \texttt{Strong-symmetry}.
\end{proposition}
\begin{proof}
	Assume the input dimension $d = 2$, the baseline $\xbf' < 0$ (element-wise) and $\xbf \in \Reals^2$.
	Using Proposition~\ref{prop:sens_complete}, we get
	\begin{align*}
		\phi(\xbf, \indic{\pi_1 \geq 0}) &= \left(\indic{\xbf_1 \geq 0}, 0 \right)^\top \, , \\
		\phi(\xbf, \indic{\pi_2 \geq 0}) &= \left(0, \indic{\xbf_2 \geq 0} \right)^\top
	\end{align*}
	By \texttt{Linearity}
	\begin{align*}
		\phi(\xbf, \indic{\pi_1 \geq 0} + \indic{\pi_2 \geq 0}) = \left(\indic{\xbf_1 \geq 0}, \indic{\xbf_2 \geq 0} \right)^\top \, .
	\end{align*}
	As $\indic{\pi_1 \geq 0} + \indic{\pi_2 \geq 0}$ is symmetric, thus, by \texttt{Strong-symmetry}
	\begin{align*}
		\phi(\xbf, \indic{\pi_1 \geq 0} + \indic{\pi_2 \geq 0})_1 = \phi(\xbf, \indic{\pi_1 \geq 0} + \indic{\pi_2 \geq 0})_2 \, .
	\end{align*}
	Using \texttt{Completeness} on $\indic{\pi_1 \geq 0} + \indic{\pi_2 \geq 0}$
	\begin{align*}
		\phi(\xbf, \indic{\pi_1 \geq 0} + \indic{\pi_2 \geq 0}) = \frac{1}{2} \left(\indic{\xbf_1 \geq 0} + \indic{\xbf_2 \geq 0}, \indic{\xbf_1 \geq 0} + \indic{\xbf_2 \geq 0} \right)^\top \, .
	\end{align*}
	The incoherency arises here as there exist $\xbf \in \Reals^2$ such that
	\begin{align*}
		\frac{1}{2} \left(\indic{\xbf_1 \geq 0} + \indic{\xbf_2 \geq 0}, \indic{\xbf_1 \geq 0} + \indic{\xbf_2 \geq 0} \right)^\top \neq  \left(\indic{\xbf_1 \geq 0}, \indic{\xbf_2 \geq 0} \right)^\top \, .
	\end{align*}
\end{proof}

\section{Appendix for Section~\ref{sec:3}} \label{app:3}
\subsection{Appendix for Section~\ref{sec:3.1}} \label{app:3.1}
\begin{figure}[!t]
  \centering
    	\begin{tikzpicture}[scale=2.3]
		
		\node[below] at (0,0) {0};
		\node[below] at (1,0) {1};
		\node[left] at (0,1) {1};
		
		\node[below] at (1/3,-0.03) {$\mathcolor{gray}{a}$};
		\node[below] at (1-1/3,0) {$\mathcolor{gray}{b}$};
		
		\node[left] at (0,1/4) {$\mathcolor{gray}{c}$};
		\node[left] at (0,3/4) {$\mathcolor{gray}{d}$};

		\draw[dotted, thick, color=gray] (1/3,0) -- (1/3,1/3);
		\draw[dotted, thick, color=gray] (1-1/3,0) -- (1-1/3,1/3);
		
		\draw[dotted, thick, color=gray] (0,1/4) -- (1/3,1/4);
		\draw[dotted, thick, color=gray] (0, 3/4) -- (1/3,3/4);
		\draw[thick, ao] (1/3,1/4) rectangle (1-1/3,3/4);

		\draw[thick] (0,0) rectangle (1,1);
		
		\node[anchor=north east] at (1/3,1/4) {$\mathcolor{red}{+}$};
		\node[red, circle, draw, fill, minimum size = 4pt, inner sep=0pt, outer sep=0pt] at (1/3,1/4) {};
		
		\node[anchor=north west] at (1-1/3,1/4) {$\mathcolor{red}{-}$};
		\node[red, circle, draw, fill, minimum size = 4pt, inner sep=0pt, outer sep=0pt] at (1-1/3,1/4) {};
		
		\node[anchor=south east] at (1/3,3/4) {$\mathcolor{red}{-}$};
		\node[red, circle, draw, fill, minimum size = 4pt, inner sep=0pt, outer sep=0pt] at (1/3,3/4) {};
		
		\node[anchor=south west] at (1-1/3,3/4) {$\mathcolor{red}{+}$};
		\node[red, circle, draw, fill, minimum size = 4pt, inner sep=0pt, outer sep=0pt] at (1-1/3,3/4) {};
	\end{tikzpicture}
  \caption{\label{fig:increment}Illustration of the increment operator $\Delta(\cdot; \cdot)$ in the 2D setting. Given a rectangle $\mathcolor{ao}{R} \defeq [a,b] \times [c, d] \subset [0,1]^2$, the 2D increment for a function $g: \Reals^2 \to \Reals$ is defined as $\Delta (g;R) \defeq \mathcolor{red}{+} g(a,c)\mathcolor{red}{-} g(b,c) \mathcolor{red}{-} g(a,d) \mathcolor{red}{+} g(b,d)$. 
}
\end{figure}
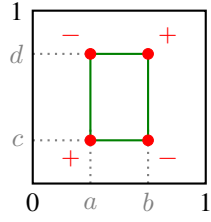
In this section, we present the approximation result which underpins our constructivist approach to feature attribution.

\begin{theorem}[Approximation of continuous function on a compact] \label{th:approx}
	Given a continuous function $f : [0,1]^d \to \Reals$ and a grid partition of $[0,1]^d$ defined as 
    \[
    \coprod_{i_1, \ldots, i_d= 0}^{p_n -1} R_{i_1, \ldots, i_d}^{(n)} = [0,1]^d \quad \text{(disjoint union)}\, 
    \] 
    with $( p_n )_n \in (\Posint)^{\mathbb{N}}$ an increasing positive sequence, a step $a^{(n)}_i \defeq \frac{i}{p_n}$ ($i \in \llbracket 0, p_n \rrbracket$) and a interior rectangle 
    \[
    R^{(n)}_{i_1, \ldots, i_d} \defeq \prod_{k=1}^{d} \left(a^{(n)}_{i_k}; a^{(n)}_{i_k+1} \right]
    \]
    for $i_1, \ldots, i_d \in \llbracket p_n-1 \rrbracket $.
	Exterior rectangles $R^{(n)}_{i_1, \ldots, i_d}$ are similarly defined but with the interval closed on the left for all indices $i_k$ equal to $0$.
	We have the following uniform approximation of $f$ using indicator functions:
	\begin{align*}
   \lim_{n \to +\infty} \norm{\left( \sum_{i_1, \ldots, i_d= 0}^{p_n-1} f\left(a^{(n)}_{i_1}, \ldots, a^{(n)}_{i_d}\right) \indic{R_{i_1, \ldots, i_d}^{(n)}}  \right)  -f}_\infty = 0 \, .
	\end{align*}
\end{theorem}
\begin{proof}
    See \cite{rudin1976principles}.
\end{proof}
\subsection{Appendix for Section~\ref{sec:3.2}} \label{app:3.2}
We introduce here the property \texttt{FPC}, a counterpart to \texttt{FSC} as defined in the main paper.
\begin{definition}[functional point-wise sequential continuity]\label{propbox:fpc}
	A feature attribution $\phi$ is \emph{functionally point-wise sequentially continuous} (\texttt{FPC}) if for all sequences $(f_n)_{n \geq 0} \in \mathcal{F}^{\Posint}$ and all $f \in \mathcal{F}$:
	\begin{align*}
		 \forall \Img \in \Reals^d,  \lim_{n \to +\infty} f_n (\Img)=f(\Img) \implies  \forall \xbf \in \Reals^d,  \lim_{n \to +\infty} \phi (\xbf, f_n) =  \phi (\xbf, f)  \, .
	\end{align*}
	The last limit $\lim_{n \to +\infty} \phi (\xbf, f_n) =  \phi (\xbf, f)$ is taken component-wise as we are working in finite dimension $d$.
\end{definition}
In other words, for all $\xbf \in \Reals^d$ and $j \in \llbracket d \rrbracket$, $\phi(\xbf, \cdot)_j$ is sequentially continuous where $\mathcal{F}$ is endowed with the topology of point-wise convergence. 
\begin{remark}
	Beware, with the topology of point-wise convergence (which is not metrizable), the sequential continuity and continuity are not equivalent.  Indeed, in this case, sequential continuity is weaker than continuity. 
\end{remark}

\subsection{Appendix for Section~\ref{sec:3.3}} \label{app:3.3}
\subsubsection{Formalization of Section~\ref{sec:3.1} using the Riemann-Stieltjes integral}
In this section, we introduce the necessary concepts to formalize and generalize the construction from Section~\ref{sec:3.1} to the $d$-dimensional input space using the Stieltjes integral.
\paragraph{Notations.}
First, we need to introduce some general notation following \cite{owen2005multidimensional}. 
Given two vectors $\abf, \bbf \in \Reals^d$, we write $\abf \leq \bbf$ \emph{iff} for all $i \in \llbracket d \rrbracket$, $\abf_i \leq \bbf_i$. Using this, we define the rectangle $[\abf, \bbf] \defeq \{ \xbf \in \Reals^d : \abf \leq \xbf \leq \bbf \} = \prod_{i=1}^{d} [\abf_i, \bbf_i]$. 
With $u,v  \subset \llbracket d \rrbracket$, we denote by $\card{u}$ the cardinality of $u$. 
Also, $-u$ denotes the complement of $u$ in $\llbracket d \rrbracket$ and more generally, $u-v$ the complement of $v$ in~$u$. 
Furthermore, for $u \subset \llbracket d \rrbracket$ and $\xbf \in [\abf, \bbf]$, $\xbf^u$ is the $\card{u}$-sized vector $\left(\xbf_{i_1}, \xbf_{i_2}, \ldots, \xbf_{i_{\card{u}}} \right)^\top$ where $i_1 < i_2 < \cdots < i_{\card{u}}$ are all the ordered elements of $u$. 
Also, for two disjoint subsets $u, v \subset \llbracket d \rrbracket$ and $\xbf, \ybf \in [\abf, \bbf]$, we define the vector gluing $\xbf^u : \ybf^v$ as a $\card{u \cup v}$-sized vector $\zbf \in [\abf^{u \cup b}, \bbf^{u \cup b}]$ with $\zbf_i = \xbf_i$ if $i \in u$, and $\zbf_i = \ybf_i$ if $i \in v$.
We give an example of vector gluing in dimension $d = 3$ with $\xbf, \ybf \in \Reals^3$ and $u \defeq \{2 \}, v \defeq \{ 1,3\}$:
\[
 \xbf^u : \ybf^v = (\ybf_1, \xbf_2, \ybf_3)\, .
\]
Finally, a one-dimensional ladder $\mathcal{Y}$ on $[a,b]$ is a finite set containing at least $a$ with values in $(a,b)$, the ladder does not contain $b$ (except in the degenerated case $a = b$). 
Each element of the ladder $y$ has a successor $y_+$, if the elements of $\mathcal{Y}$ are arranged as $a=y_0 < y_1 < y_2 < \cdots < y_{\card{\mathcal{Y}}}$ , then the successor of $y_k$ is $y_{k+1}$ (when $k < \card{\mathcal{Y}}$) and if $k = \card{\mathcal{Y}}$, then the successor is $b$ (beware, the dependencies of $y_+$ on $\mathcal{Y}$ is not made explicit).  
A multidimensional ladder $\mathcal{Y}$ on $[\abf, \bbf]$ has the form $\mathcal{Y} \defeq \prod_{i=1}^{d} \mathcal{Y}^i$ where $\mathcal{Y}^i$s are one-dimensional ladders on $[\abf_i, \bbf_i]$. A successor $\ybf_+$ of $\ybf \in \mathcal{Y}$ is defined by taking $(\ybf_+)_i$ as the successor of $\ybf_i$ in $\mathcal{Y}^i$. The set of all ladders in $[\abf, \bbf]$ is $\mathbb{Y}([\abf, \bbf])$.
We denote by $\mathcal{B}([0,1]^d)$ the Borel $\sigma$-algebra of $[0,1]^d$.

We start with an informal description of the Stieltjes integral in the $1$D case:
given two real‐valued functions \(f\) and \(g\) on an interval \([a,b]\), the Stieltjes integral
\[
\int_a^b f(x)\,\Diff g(x)
\]
is the limit (when it exists) of the weighted sum
\[
\sum_{i=1}^n f(\xi_i)\,\bigl[g(x_i)-g(x_{i-1})\bigr],
\]
where \(a=x_0<x_1<\cdots<x_n=b\) is a partition of \([a,b]\) and each \(\xi_i\in[x_{i-1},x_i]\).
Briefly, to define multidimensional Stieltjes integral, we need to generalize the 1-dimensional increment $g(x+) -~g(x)$, one way to do it is:
\begin{definition}[Increment operators, \cite{owen2005multidimensional}] \label{def:increment-operators}
 Given $f :  [0, 1]^d \to \Reals$, we define the increment operator as: 
 \[
 \forall \abf, \bbf \in [0,1]^d, \qquad \Delta (f; \abf, \bbf) \defeq \sum_{v \subset \llbracket d \rrbracket} (-1)^{\card{v}} f( \abf^{v} : \bbf^{-v}) \, .
 \] 
 A useful generalization of the previous operator, with a subset $u \subset \llbracket d \rrbracket$, is:
  \[
 \forall \abf, \bbf \in [0,1]^d, \qquad\Delta_{u} (f ; \abf, \bbf) \defeq \sum_{v \subset u} (-1)^{\card{v}} f( \abf^{v} : \bbf^{u-v} : \abf^{-u}) \, .
 \] 
 We will also use the notation $\Delta (f, R)$ interchangeably with $\Delta (f ; \abf, \bbf)$ where $R = (\abf, \bbf] \subset [0,1]^d$.
\end{definition}
In other words, one way to generalize the 1-dimensional increment is to define a sign-alternated sum on the vertices of hyperrectangles as depicted in Figure~\ref{fig:increment}.
One cannot integrate w.r.t. to any integrand~$g$, indeed it need to be of bounded variation (roughly speaking, the function does not oscillate to much).
We introduce two multidimensional extensions of one-dimensional bounded variation, namely Vitali bounded variation (\texttt{BV}) and Hardy-Krause bounded variation (\texttt{BHK}), as follows
\begin{definition}[Vitali and Hardy-Krause variations, \cite{owen2005multidimensional}] \label{app:hk-variation}
	Given $\abf, \bbf \in \Reals^d$, a function $g~:~[\abf,\bbf] \to \Reals$ has Vitali bounded variation (abbreviated as \texttt{BV}) if the following quantity is bounded:
	\[
	V_{[\abf,\bbf]}(g) \defeq \sup_{\mathcal{Y} \in \mathbb{Y}([\abf,\bbf])}	\sum_{\ybf \in \mathcal{Y}} \abs{\Delta(g; \ybf, \ybf_+)}	 < +\infty \, .
	\]
	$V_{[\abf,\bbf]}(g)$ is the Vitali variation of $g$ on $[\abf,\bbf]$.
	
	Similarly, $g$ have Hardy-Krause bounded variation (abbreviated as \texttt{BHK}) at $\mathbf{1}$ if:
	\[
	V_{HK} (g) \defeq \sum_{u \subsetneq \llbracket d \rrbracket}  V_{[\abf^{-u},\bbf^{-u}]}(g((\cdot)^{-u}, \bbf^u)) < +\infty \, .
	\]
	$V_{HK} (g)$ is the Hardy-Krause variation of $g$ on $[\abf,\bbf]$.
\end{definition}
The Hardy-Krause variation is more refined than the Vitali variation because it takes into account variation over all lower-dimensional subspaces.
The following definition provides one possible generalization of monotonicity to higher dimensions.
\begin{definition}[Completely monotone]
    A function $g:[0,1]^d \to \Reals$ is \emph{completely monotone} if:
    \[
    \forall \xbf, \ybf \in [0,1]^d, \qquad \left( \xbf \leq \ybf \implies \forall u \subset \llbracket d \rrbracket , \, \Delta_u (g; \xbf, \ybf) \geq 0 \right) \, .
    \]
\end{definition}
Now, the construction of Section~\ref{sec:3.1} is made rigorous and generalized to general input dimension $d$, using the Stieltjes integral, in the following theorem:
\begin{theorem}[General increment-based atomic attribution]\label{th:increment-based-atomic-attribution}
	Let $\phi$ be a \texttt{Linear} and \texttt{FSC} attribution method, and $\{ g_{j, \xbf} \}_{j \in \llbracket d \rrbracket} \subset \{ g \mid g : \Reals^d \to \Reals \}$ be a family of \texttt{BV} functions with $\xbf \in [0,1]^d$.
	Assume that the atomic attributions of indicator functions are given by:
	\begin{align*}
		\forall \xbf \in [0,1]^d, \forall j \in \llbracket d \rrbracket, \qquad
		\phi\left(\xbf, \indic{R} \right)_j \defeq 
		\Delta  \left(g_{j, \xbf}; R \right)  \in \Reals \, ,
	\end{align*}
	where $R \subset [0,1]^d$ is a right-closed hyperrectangle.
	Then, the attribution of a continuous model $f: [0,1]^d \to \Reals$ is the following Riemann-Stieltjes integral:
	\begin{align*}
		\forall \xbf \in [0,1]^d, \forall j \in \llbracket d \rrbracket, \qquad \phi\left(\xbf, f \right)_j = \int_{[0,1]^d} f(\ybf) \, \Diff g_{j, \xbf}(\ybf)\, .
	\end{align*}
\end{theorem}
\begin{proof}
	Given a continuous model $f : [0,1]^d \to \Reals$. We apply \texttt{Linearity} to the approximation given by Theorem~\ref{th:approx}:
	
	\begin{align*}
		\forall \xbf \in \Reals^d, \qquad &\phi\left(\xbf,   \sum_{i_1, \ldots, i_d= 0}^{p_n-1} f\left(a^{(n)}_{i_1}, \ldots, a^{(n)}_{i_d}\right) \indic{ R_{i_1, \ldots, i_d}^{(n)}} \right) \\
		&= \sum_{i_1, \ldots, i_d= 0}^{p_n-1}  f\left(a^{(n)}_{i_1}, \ldots, a^{(n)}_{i_d}\right)  \phi\left(\xbf,  \indic{R_{i_1, \ldots, i_d}^{(n)}}\right) \\
		&= \sum_{i_1, \ldots, i_d= 0}^{p_n-1} f\left(a^{(n)}_{i_1}, \ldots, a^{(n)}_{i_d}\right)   \left(\Delta  \left(g_{1, \xbf}, R^{(n)}_{i_1, \ldots, i_d} \right), \ldots, \Delta  \left(g_{d, \xbf} , R^{(n)}_{i_1, \ldots, i_d} \right)  \right)^\top  \, .
	\end{align*}
We can work coordinate-wise,
\begin{align}
	\begin{split}
		\forall \xbf \in \Reals^d, \forall j \in \llbracket d \rrbracket, \qquad &\phi\left(\xbf,   \sum_{i_1, \ldots, i_d= 0}^{p_n-1} f\left(a^{(n)}_{i_1}, \ldots, a^{(n)}_{i_d}\right) \indic{R_{i_1, \ldots, i_d}^{(n)}} \right)_j \\
		&= \sum_{i_1, \ldots, i_d= 0}^{p_n-1} f\left(a^{(n)}_{i_1}, \ldots, a^{(n)}_{i_d}\right)   \Delta  \left(g_{j, \xbf}, R^{(n)}_{i_1, \ldots, i_d} \right) 
   \, .
	\end{split}
\end{align}
One recognizes the multiple Riemann-Stieltjes sum of Theorem~\ref{th:general-riemann-stieltjes}:
\begin{align}
	\begin{split} \label{eq:stieltjes-proof-1}
			\forall \xbf \in \Reals^d, \forall j \in \llbracket d \rrbracket, \qquad 
		&\lim_{n \to +\infty} \phi\left(\xbf,    \sum_{i_1, \ldots, i_d= 0}^{p_n-1} f\left(a^{(n)}_{i_1}, \ldots, a^{(n)}_{i_d}\right) \indic{R_{i_1, \ldots, i_d}^{(n)}} \right)_j   \\
		&= \int_{[0,1]^d} f(t_1,t_2, \ldots, t_d) \, \Diff g_{j, \xbf}(t_1, t_2, \ldots, t_d)\, .
	\end{split}
\end{align}
Finally, as $\phi$ is \texttt{FPC} (or \texttt{FSC}) and we have the following convergence
\[
\lim_{n \to +\infty} \norm{\left( \sum_{i_1, \ldots, i_d= 0}^{p_n-1} f\left(a^{(n)}_{i_1}, \ldots, a^{(n)}_{i_d}\right) \indic{R_{i_1, \ldots, i_d}^{(n)}}  \right)  -f}_\infty = 0 \, .
\]
Then
\begin{align*}
	\forall \xbf \in \Reals^d, \qquad \lim_{n \to +\infty} \phi\left(\xbf,   \sum_{i_1, \ldots, i_d= 0}^{p_n-1} f\left(a^{(n)}_{i_1}, \ldots, a^{(n)}_{i_d}\right) \indic{R_{i_1, \ldots, i_d}^{(n)}} \right) = \phi(\xbf, f) \, ,
\end{align*}
where the limit coordinates are Riemann-Stieltjes integrals as in Equation~(\ref{eq:stieltjes-proof-1}).
\end{proof}
One should read Th.~\ref{th:increment-based-atomic-attribution} as follows: \textbf{first, choose an attribution for the indicator function $\phi\left(\xbf, \indic{R} \right)$, then, get the attribution for the whole model $\phi\left(\xbf, f \right)$ in the form of an integral.}
The intuition behind the integrands $g_{j, \xbf}$ is not always immediately clear. 
For this reason, we have reformulated the result in terms of measures in the main paper (Theorem~\ref{th:measure-atomic-attribution}).
As stated in the following technical result, the relationship between $\Delta(g_{j, \xbf}; [\mathbf{0}, \cdot])$ and measures is bijective, allowing us to think in terms of measures rather than integrands. 
\begin{theorem}[Signed Borel measure and bounded variation, \cite{aistleitner2014functions}]\label{th:signed-borel}
	Let $h: [0,1]^d \to \Reals$ be a \texttt{BHK}, and coordinate-wise right-continuous function. 
	Then, there exists a unique signed Borel measure $\mu$ on $[0,1]^d$ such that
	\begin{align}
	\forall \xbf \in [0,1]^d , \qquad  h(\xbf) &= \mu([\mathbf{0}, \xbf]) \, , \\
	\mathrm{Var}_{total}(\mu) &= V_{HK}(f) + \abs{f(\mathbf{0})}\, ,
	\end{align}
	where $\mathrm{Var}_{total}(\cdot)$ is the total variation of a measure defined as $\mathrm{Var}_{total}(\mu) \defeq \mu^+([0,1]^d) + \mu^-([0,1]^d)$ with $(\mu^+, \mu^-)$ the positive measures from the Jordan decomposition of $\mu$.
	
	Conversely, if $\mu$ is finite signed Borel measure on $[0,1]^d$. Then, there exists a unique \texttt{BHK}, and coordinate-wise right-continuous function $h$ on $[0,1]^d$ that satisfies $(1.4)$ and $(1.5)$.
\end{theorem}
\begin{proof}
    See \cite{Breneis2020}.
\end{proof}
We now state a corollary of the previous result that yields positive measures.
\begin{corollary}[Positive Borel measure and bounded variation]\label{cor:positive-borel}
	Let $h: [0,1]^d \to \Reals$ be a \texttt{BHK}, coordinate-wise right-continuous, and completely monotone function. 
Then, there exists a unique positive Borel measure $\mu$ on $[0,1]^d$ such that
\begin{align*}
	\forall \xbf \in [0,1]^d , \qquad  h(\xbf) &= \mu([\mathbf{0}, \xbf]) \, , \\
	\mathrm{Var}_{total}(\mu) &= V_{HK}(f) + \abs{f(\mathbf{0})}\, ,
\end{align*}
Conversely, if $\mu$ is finite positive Borel measure on $[0,1]^d$. Then, there exists a unique \texttt{BHK}, coordinate-wise right-continuous, and completely monotone function $h$ on $[0,1]^d$ that satisfies $(1.4)$ and $(1.5)$.
\end{corollary}
\begin{proof}
    See \cite{Breneis2020}. It is a direct consequence of completely monotone and Theorem~\ref{th:signed-borel}.
\end{proof}
\begin{remark}
	Right-continuous is needed (as a convention) to properly define the measure.
\end{remark}
Using this result, we rewrite Th.~\ref{th:general-riemann-stieltjes} with measures by assuming that the integrands are \texttt{BHK} (instead of \texttt{BV}) and right-continuous.

\subsubsection{Remainder on the Riemann-Stieltjes integral}
For the sake of completeness, we recall the definition of the Riemann–Stieltjes integral. 
It suffices for the integrand to have bounded Vitali variation for the integral to be well-defined.
\begin{theorem}[Riemann-Stieltjes integral in $\Reals^d$, \cite{hildebrandt1963introduction}]\label{th:general-riemann-stieltjes}
	Let $f: [0,1]^d \to \Reals$ be a continuous function and $g: [0,1]^d \to \Reals$ be a \texttt{BV} function. The following Riemann-Stieltjes sum converges and its limit is the Riemann-Stieltjes integral:
	\begin{align*}
		&\lim_{n \to +\infty} \sum_{i_1, \cdots, i_d = 0}^{p_n-1} f\left(a^{(n)}_{i_1}, \ldots, a^{(n)}_{i_d} \right) \Delta \left(g, R^{(n)}_{i_1, \ldots, i_d} \right)   \\
		&= (\mathcal{RS}) \int_{[0,1]^d} f(t_1,t_2, \ldots, t_d) \, \Diff g(t_1, t_2, \ldots ,t_d) \, ,
	\end{align*}
	where the rectangle $R_{i_1, \ldots, i_d}^{(n)}$ comes from Theorem~\ref{th:approx} and the increment operator $\Delta (g, \cdot)$ can be written as:
	\begin{align*}
		\Delta \left(g, R^{(n)}_{i_1, \ldots, i_d} \right) =  \sum_{\epsilonbm \in \{0,1\}^d} (-1)^{\norm{\epsilonbm}_1} g\left(a^{(n)}_{i_1 + \epsilonbm_1}, a^{(n)}_{i_2 + \epsilonbm_2}, \ldots, a^{(n)}_{i_d + \epsilonbm_d}\right)    \, .
	\end{align*}
\end{theorem}
\begin{proof}
	See~\cite{hildebrandt1963introduction}.
\end{proof}
One specific choice of integrands that gives rise to the multiple Riemann–Stieltjes integral is presented below.
\begin{corollary}[multiple Riemann-Stieltjes integral in $\Reals^d$] \label{cor:multiple-riemann-stieltjes}
	Let $f: [0,1]^d \to \Reals$ be a continuous function  and $\{ g_{i} \}_{i \in \llbracket d \rrbracket} \subset \Reals^\Reals$ be a family  of non-constant, and \texttt{BV} functions. 
    The following Riemann-Stieltjes sum converges and its limit is the multiple Riemann-Stieltjes integral:
	\begin{align*}
			&\lim_{n \to +\infty} \sum_{i_1, \cdots, i_d = 0}^{p_n-1} f(a^{(n)}_{i_1}, \cdots, a^{(n)}_{i_d}) \left(g_{1}\left(a^{(n)}_{i_1+1}\right) - g_{1}\left(a^{(n)}_{i_1}\right) \right) \left(g_{2}\left(a^{(n)}_{i_2+1}\right) - g_{2}\left(a^{(n)}_{i_2}\right) \right) \cdots  \\
            &\cdots\left(g_{d}\left(a^{(n)}_{i_d+1}\right) - g_{d}\left(a^{(n)}_{i_d}\right) \right) \\
			&= \int_{[0,1]} \cdots \int_{[0,1]} f(t_1,t_2, \ldots, t_d) \, \Diff g_{1}(t_1) \Diff g_{2}(t_2)\cdots  \Diff g_{d}(t_d) \, ,
	\end{align*}
	where the partition element $a^{(n)}_{i}$ comes from Theorem~\ref{th:approx}.
	Also, Fubini's theorem is applicable to the previous multiple integral.
\end{corollary}
\begin{proof}
	See~\cite{hildebrandt1963introduction} (apply Theorem~\ref{th:general-riemann-stieltjes} to $g(t_1, t_2, \ldots, t_d) \defeq g_1(t_1)g_2(t_2) \cdots g_d(t_d)$).
\end{proof}
One very interesting integrands are the indicator functions, for which we compute the Riemann-Stieltjes integral explicitly as 
\begin{proposition}[Riemann-Stieltjes integral with indicator function integrand]\label{prop:step-function-integrand}
	Let $f: [0,1] \to \Reals$ be a continuous function and $g: [0,1] \to \Reals$ be a indicator function with $n$ discontinuities at $0 \leq  t_1 < t_2 < \cdots < t_n \leq 1$. The following Riemann-Stieltjes integral is equal to:
	\begin{align*}
		 \int_{[0,1]} f(x) \, \Diff g(x)   
		= \sum_{i = 1}^{n} f(t_i) \left(g(t_i+) - g(t_i-) \right) \, , 
	\end{align*}
	where $g(t+) = \lim_{x \to t, x>t} g(x)$ and $g(t-) = \lim_{x \to t, x<t} g(x)$ with the convention that $g(t_1-) = g(0)$ and  $g(t_n+) = g(1)$ if $t_1 = 0$ or $t_n = 1$.
\end{proposition}

\begin{proof}
	The proof can be found online.\footnote{\url{https://personal.math.ubc.ca/~feldman/m321/step.pdf}}
\end{proof}
\subsubsection{Measure-based feature attribution}
One might ask why we don't directly use general measure-based attributions as in Th.~\ref{th:measure-atomic-attribution}. The answer is that this approach does not directly fit within the approximation result of Th.~\ref{th:approx} (which is a classical construction for the Riemann integral). In a Lebesgue-type integral, one partitions the range (image space) of the function and then takes the preimages of these partitions. However, these preimages do not necessarily form hyperrectangles (which are very simple objects), unlike the straightforward case where the partition is applied directly to the input space. 
Moreover, the Stieltjes integrand-based framework is more general, as it works with \texttt{BV} integrands $g$ rather than requiring \texttt{BHK} and right-continuity, the latter being a stronger condition than \texttt{BV}.
\begin{remark} 
	Although, one can derive the same framework using the Lebesgue integral construction, doing so requires additional work. In particular, it involves rewriting the preimages of the range partition elements as unions of hyperrectangles. 
\end{remark}
Finally, we clarify the statement made at the end of Section~\ref{sec:3.3} by formalizing the density-based feature attribution framework.
\begin{corollary}[Radon-Nikodym feature attribution] \label{cor:nikodym-attribution}
	Let $\xbf \in [0,1]^d$ and $\{\mu_{j, \xbf}\}_{j \in \llbracket d \rrbracket}$ be the family of finite signed Borel measures which are absolutely continuous w.r.t. the Lebesgue measure with respective Radon-Nikodym derivatives $\{h_{j, \xbf}\}_{j \in \llbracket d \rrbracket}$, meaning 
	\[
	\forall j \in \llbracket d \rrbracket, \forall A \in \mathcal{B}([0,1]^d), \qquad \mu_{j, \xbf}(A) = \int_{A} h_{j, \xbf}(\ybf) \, \Diff \ybf \, .
	\]
	Let $\phi$ be \texttt{Linear} and \texttt{FPC} (or \texttt{FSC}) attribution method.
	Assume that the atomic attributions of the indicator functions are given by:
	\begin{align*}
		\forall \xbf \in [0,1]^d, \forall j \in \llbracket d \rrbracket, \qquad
		\phi\left(\xbf, \indic{R} \right)_j \defeq 
		\mu_{j, \xbf}\left(R\right)  \in \Reals \, ,
	\end{align*}
	where $R \subset [0,1]^d$ is a hyperrectangle.
	Then, the attribution of a continuous model $f: [0,1]^d \to \Reals$ is the following Lebesgue-Stieltjes integral:
	\begin{align*}
		\forall \xbf \in [0,1]^d, \forall j \in \llbracket d \rrbracket, \qquad \phi\left(\xbf, f \right)_j =  \int_{[0,1]^d} f(\ybf) \, \Diff \mu_{j, \xbf}(\ybf) = \int_{[0,1]^d} f(\ybf) h_{j, \xbf}(\ybf) \, \Diff \ybf \, .
	\end{align*}
\end{corollary}
\begin{proof}
It is a direct consequence of Theorem~\ref{th:measure-atomic-attribution}.
\end{proof}

\section{Appendix for Section~\ref{sec:4}} \label{app:4}
\subsection{Appendix for Section~\ref{sec:4.1}} \label{app:4.1}

We define deep ReLU networks, as referred to in the main paper.
\begin{definition}[Deep ReLU network]\label{def:relu-net}
	 Let $L \in \Posint$, $\Wbf \defeq \left\{ \Wbf^{(\ell)} \right\}_{\ell \in \llbracket L+1 \rrbracket}$, and $\bbf \defeq \left\{ \bbf^{(\ell)} \right\}_{\ell \in \llbracket L+1 \rrbracket}$, where, for each layer $\ell \in \llbracket L+1 \rrbracket$, the weight matrix \(\Wbf^{(\ell)} \in \Reals^{d_\ell \times d_{\ell-1}}\) (with \(d_0 = d\) and \(d_{L+1} = 1\)) and the bias vector \(\bbf^{(\ell)} \in \Reals^{d_\ell}\).
	 We define a $L$-hidden layer deep ReLU network $f: \Reals^d \to \Reals$ as follows:
	\[
	\forall \xbf \in \Reals^d, \qquad f(\xbf) = \Wbf^{(L+1)} \sigma\left( \Wbf^{(L)} \sigma\left( \cdots \sigma\left( \Wbf^{(1)} \xbf + \bbf^{(1)} \right) \cdots \right) + \bbf^{(L)} \right) + \bbf^{(L+1)}\, ,
	\]
	where $x \mapsto \sigma(x) = \max{(0,x)}$ is the ReLU function applied element-wise.
	
	We also define recursively the pre-activation $\zbf^{(\ell)}$ of layer $\ell \in \llbracket L +1 \rrbracket$ for the input $\xbf \in \Reals^d$ as 
	\[
	\zbf^{(1)} \defeq \Wbf^{(1)} \xbf + \bbf^{(1)} \quad \text{and} \quad \zbf^{(\ell)} \defeq \Wbf^{(\ell)} \sigma\left( \zbf^{(\ell-1)} \right) + \bbf^{(\ell)}  \quad  \, . 
	\] 
\end{definition}
Using the partition of the input space introduced in Definition~\ref{def:piecewise-affine}, we present the following representation theorem for deep ReLU networks:
\begin{theorem}[Linear representation of Deep ReLU networks, \cite{humayun2023splinecam}] \label{th:rep-relu}
	Given a deep ReLU network $f: [0,1]^d \to \Reals$ and the partition of the input space  $\mathcal{R}([0,1]^d)$ forming the decision boundary of the network, we write $f$ as:
	\[
	\forall \xbf \in [0,1]^d, \qquad f(\xbf) = \sum_{P \in \mathcal{R}([0,1]^d)} \left(\abf_{P}^\top \xbf + b_P \right) \indic{\xbf \in P} \, ,
	\]
	where $\abf_{P} \in \Reals^d$ and $b_P \in \Reals$ are the coefficients of the local linear model in region $P$ defined as
	\begin{align*}
	 \abf_{P} &\defeq \left( \prod_{\ell = 1}^{L} {\Wbf^{(\ell)}}^\top \diag{\indic{\zbf^{(\ell)} > 0} }  \right)  {\Wbf^{(L+1)}}^\top  \, , \\
	 b_P &\defeq b^{(L+1)} + \sum_{\ell =  1}^{L} \left(  \prod_{j = \ell+1}^{L} {\Wbf^{(j)}}^\top\diag{\indic{\zbf^{(j)} > 0} }  \right){\Wbf^{(L+1)}}^\top \diag{\indic{\zbf^{(\ell)} > 0} }  \bbf^{(\ell)} \, .
 	\end{align*}
 	The product in the definition of $b_P$ is equal to identity (by convention) if $\ell = L$ and $\indic{\zbf^{(\ell)} > 0}   \in \Reals^{d_\ell}$ is the point-wise derivative of ReLU at the pre-activation $\zbf^{(\ell)}$.
 	In the simple case, with no bias, $b_P = 0$.
\end{theorem}
The following result is the general version of Cor.~\ref{cor:linear-positive}.
\begin{corollary}[Feature attribution of piecewise affine continuous function with signed measures]\label{cor:relu-attrib}
	Let $\phi$ be a \texttt{Linear} and \texttt{FPC} (or \texttt{FSC}) attribution method, and $\{ g_{j, \xbf} \}_{j \in \llbracket d \rrbracket, \xbf \in [0,1]^d}$ be a family of coordinate-wise right-continuous and \texttt{BHK} functions.
Assume that the atomic attributions of indicator functions are given by:
\begin{align*}
	\forall \xbf \in [0,1]^d, \forall j \in \llbracket d \rrbracket, \qquad
	\phi\left(\xbf, \indic{R} \right)_j \defeq 
	\Delta \left(g_{j, \xbf}, R \right)  \in \Reals \, ,
\end{align*}
where $R  \subset [0,1]^d$ is a right-closed hyperrectangle.
Then, the attribution of a piecewise affine continuous function $f: [0,1]^d \to \Reals$ is:
\begin{align*}
	\forall \xbf \in [0,1]^d, \forall j \in \llbracket d \rrbracket, \qquad \phi\left(\xbf, f \right)_j = \sum_{P \in \mathcal{R}_{0,j}^\complement} \mu_{j, \xbf} (P) \left( \abf_{P}^\top \mathbf{m}^{(P, j, \xbf)} + b_P \right) + 
	\sum_{P \in \mathcal{R}_{0,j}} \abf_{P}^\top \mathbf{n}^{(P, j, \xbf)}   \, ,
\end{align*}
where $\mu_{j, \xbf}(\cdot)$ is the signed Borel measure associated to $\Delta (g_{j, \xbf}, \cdot)$, the set of non-zero regions $\mathcal{R}_{0,j} \defeq \{ P \in \mathcal{R}([0,1]^d) : \mu_{j, \xbf}(P) \neq 0 \}$ and $\mathbf{m}^{(P, j, \xbf)} \in \Reals^d$ is the mass center of $P$ defined as:
\[
	\forall i \in \llbracket d \rrbracket, \qquad \mathbf{m}^{(P, j, \xbf)}_i \defeq \frac{ \mathbf{n}^{(P, j, \xbf)}_i }{\mu_{j, \xbf}(P)} \,\,\text{ and }\,\, \mathbf{n}^{(P, j, \xbf)}_i \defeq \int_{P} \ybf_i \, \Diff \mu_{j, \xbf} (\ybf) \, .
\]
Also, if the partition $\mathcal{R}([0,1]^d)$ consists only of convex elements, then
\[
\forall \xbf \in [0,1]^d, \forall j \in \llbracket d \rrbracket, \qquad \phi\left(\xbf, f \right)_j = \sum_{P \in \mathcal{R}([0,1]^d)} \mu_{j, \xbf} (P) f\left(\mathbf{m}^{(P, j, \xbf)}\right)  \, . 
\]
\end{corollary}	
\begin{proof}
Given a deep ReLU network $f$, we get by Theorem~\ref{th:increment-based-atomic-attribution}
\begin{align*}
\forall \xbf \in [0,1]^d, \forall j \in \llbracket d \rrbracket, \qquad
\phi\left(\xbf, f \right)_j 
 &= \int_{[0,1]^d} f(\ybf) \, \Diff g_{j, \xbf}(\ybf) \, .
\end{align*}
By Theorem~\ref{th:signed-borel}, one can assign a signed Borel measure $\mu_{j, \xbf}$ to $\Delta(g_{j, \xbf}, [\mathbf{0}, \cdot])$ as it is (coordinate-wise) right-continuous \texttt{BHK} because $g_{j, \xbf}(\cdot)$ is right-continuous \texttt{BHK}. Now, the previous Riemann-Stieltjes integral can be written as a Lebesgue-Stieltjes integral (as indeed both exist and have the same value on the compact $[0,1]^d$), see~\cite{folland1999real}.
\begin{align*}
\phi\left(\xbf, f \right)_j &= \int_{[0,1]^d} f(\ybf) \, \Diff \mu_{j, \xbf}(\ybf)
\end{align*}
We apply the representation Theorem~\ref{th:rep-relu} on $f$.
\begin{align*}
\phi\left(\xbf, f \right)_j &= \int_{[0,1]^d} \sum_{P \in \mathcal{R}([0,1]^d)} \left({\abf_{P}}^\top \ybf + b_P \right) \indic{\ybf \in P}   \, \Diff \mu_{j, \xbf}(\ybf)
\end{align*}
By linearity of the integral:
\begin{align*}
\phi\left(\xbf, f \right)_j &= \sum_{P \in \mathcal{R}([0,1]^d)} \left( \sum_{i \in \llbracket d \rrbracket} \int_{P} (\abf_{P})_i \ybf_i  \, \Diff \mu_{j, \xbf}(\ybf) + \int_{[0,1]^d}  b_P \indic{\ybf \in P} \, \Diff \mu_{j, \xbf}(\ybf) \right) \\
 &=  \sum_{P \in \mathcal{R}([0,1]^d)} \left( \sum_{i \in \llbracket d \rrbracket}  (\abf_{P})_i \int_{P} \ybf_i  \, \Diff \mu_{j, \xbf}(\ybf) + b_P \int_{[0,1]^d}  \indic{\ybf \in P}\, \Diff \mu_{j, \xbf}(\ybf) \right)
\end{align*}
As $P \in \mathcal{R}([0,1]^d)$ is Borel measurable:
\begin{align*}
\phi\left(\xbf, f \right)_j &= \sum_{P \in \mathcal{R}([0,1]^d)} \left( {\abf_{P}}^\top \mathbf{n}^{(P, j, \xbf)} + b_P \mu_{j, \xbf}(P) \right)
\end{align*}
Using the disjoint decomposition $\mathcal{R}([0,1]^d) = \mathcal{R}_{0,j}^\complement \coprod \mathcal{R}_{0,j}$ (\emph{i.e.,} disjoint union):
\begin{align*}
\phi\left(\xbf, f \right)_j &= \sum_{P \in \mathcal{R}_{0,j}^\complement} \left( {\abf_{P}}^\top \mathbf{n}^{(P, j, \xbf)} + b_P \mu_{j, \xbf}(P) \right)  + \sum_{P \in \mathcal{R}_{0,j}} \left( {\abf_{P}}^\top \mathbf{n}^{(P, j, \xbf)} + b_P \mu_{j, \xbf}(P) \right) \\
&= \sum_{P \in \mathcal{R}_{0,j}^\complement} \mu_{j, \xbf} (P) \left( \abf_{P}^\top \mathbf{m}^{(P, j, \xbf)} + b_P \right) + 
\sum_{P \in \mathcal{R}_{0,j}} \abf_{P}^\top \mathbf{n}^{(P, j, \xbf)} \, .
\end{align*}
\end{proof}

\subsection{Appendix for Section~\ref{sec:4.2}} \label{app:4.2}

A natural way for feature attribution of a linear regression model $\xbf \mapsto \wbf^\top \xbf$ is to look at the learned coefficients $\wbf_j$ (global attribution) or the input-scaled coefficients $\wbf_j \xbf_j$ (local attribution). In what follows, we show that for a specific choice of the integrands, we can recover this coefficients. 
First, let us reformulate Theorem~\ref{th:general-riemann-stieltjes} for linear models.
\begin{corollary}[General feature attribution of a linear model] \label{cor:linear-model-general}
	Let $\phi$ be a \texttt{Linear} and \texttt{FPC} (or \texttt{FSC}) attribution method, and $\{ g_{j, \xbf} \}_{j \in \llbracket d \rrbracket, \xbf \in [0,1]^d}$ be a family of \texttt{BV} functions.
	Assume that the atomic attributions of indicator functions are given by:
\begin{align*}
	\forall \xbf \in [0,1]^d, \forall j \in \llbracket d \rrbracket, \qquad
	\phi\left(\xbf, \indic{R} \right)_j \defeq 
	\Delta  \left(g_{j, \xbf}, R \right)  \in \Reals \, ,
\end{align*}
where $R  \subset [0,1]^d$ is a right-closed hyperrectangle.
Then, the attribution of a linear regression model $f_\wbf: [0,1]^d \to \Reals$ ($\wbf \in \Reals^d$) is:
\begin{align*}
	\forall \xbf \in [0,1]^d, \forall j \in \llbracket d \rrbracket, \qquad \phi\left(\xbf, f_{\wbf} \right)_j = \wbf^\top \nbf^{(j, \xbf)} \, ,
\end{align*}
with $\nbf^{(j, \xbf)} \in \Reals^d$ a pre-center of mass of $[0,1]^d$ defined as:
\[
\forall i \in \llbracket d \rrbracket , \qquad \mathbf{n}^{(j, \xbf)}_i \defeq \int_{[0,1]^d} \ybf_i \, \Diff g_{j, \xbf} (\ybf) \, .
\]
\end{corollary}
\begin{proof}
	We apply Theorem~\ref{th:increment-based-atomic-attribution} as follows:
	\begin{align*}
		\forall \xbf \in [0,1]^d, \forall j \in \llbracket d \rrbracket, \qquad \phi\left(\xbf, f_{\wbf} \right)_j &=  \int_{[0,1]^d} f_\wbf(\ybf) \, \Diff g_{j, \xbf}(\ybf) \\
		&=  \int_{[0,1]^d} \wbf^\top \ybf \, \Diff g_{j, \xbf}(\ybf) \\
		&=  \sum_{i \in \llbracket d \rrbracket} \wbf_i \int_{[0,1]^d} \ybf_i \, \Diff g_{j, \xbf}(\ybf) \\
		&=  \sum_{i \in \llbracket d \rrbracket} \wbf_i  \mathbf{n}^{(j, \xbf)}_i = \wbf^\top \nbf^{(j, \xbf)} \, .
	\end{align*}
\end{proof}
The following result recovers the coefficient $\wbf_j$ as attribution of feature $j$.
\begin{corollary}[A global feature attribution for linear models] \label{cor:global-linear}
Set the integrands
\[
\forall j \in \llbracket d \rrbracket , \forall \ybf \in \Reals^d, \qquad g_{j}(\ybf) \defeq 2 \ybf_j \prod_{i \neq j} \indic{\ybf_i> 0} \, .
\]
Let $\phi$ be a \texttt{Linear} and \texttt{FPC} (or \texttt{FSC}) attribution method.
Assume that the atomic attributions of indicator functions are given by:
\begin{align*}
	\forall \xbf \in [0,1]^d, \forall j \in \llbracket d \rrbracket, \qquad
	\phi\left(\xbf, \indic{R} \right)_j \defeq 
	\Delta  \left(g_{j}, R \right)  \in \Reals \, ,
\end{align*}
where $R  \subset [0,1]^d$ is a right-closed hyperrectangle.
Then, the attribution of a linear regression model $f_{\wbf}: [0,1]^d \to \Reals$ with parameters $\wbf \in \Reals^d$ is:
\begin{align*}
	\forall \xbf \in [0,1]^d, \forall j \in \llbracket d \rrbracket, \qquad \phi\left(\xbf, f_{\wbf} \right)_j =  \wbf_j \, .
\end{align*}
\end{corollary}
\begin{proof}
	By Corollary~\ref{cor:linear-model-general}, the attribution of $f_\wbf $ is
\begin{align*}
	\forall \xbf \in [0,1]^d, \forall j \in \llbracket d \rrbracket, \qquad \phi\left(\xbf, f_{\wbf} \right)_j = \wbf^\top \nbf^{(j)} =  \sum_{k \in \llbracket d \rrbracket} \wbf_k \nbf^{(j)}_k \, .
\end{align*}
We have the following computations for the pre-center of mass
\begin{align*}
	\nbf^{(j)}_k &\defeq  \int_{[0,1]^d} \ybf_k \, \Diff g_{j} (\ybf) \\
	&=  \int_{[0,1]} \int \cdots \int \ybf_k \, \Diff(2 \ybf_j) \bigotimes_{i \neq j} \Diff (\indic{\ybf_i > 0}) &\text{(by Corollary~\ref{cor:multiple-riemann-stieltjes}.)} \\
	&= \begin{cases}
		\left(\prod_{p \neq j} \int_{[0,1]} 1 \, \Diff (\indic{\ybf_p > 0})\right) \int_{[0,1]} \ybf_k \, \Diff (2 \ybf_k)  & \text{if } k = j \, , \\
	\int_{[0,1]} \ybf_k \, \Diff (\indic{\ybf_k > 0}) 	\int_{[0,1]} 1 \, \Diff (2 \ybf_j)  \prod_{p \neq j,k} \int_{[0,1]} 1 \, \Diff (\indic{\ybf_p > 0}) & \text{else.}
			\end{cases}  &\text{(Fubini's Theorem)}\\
	&= \begin{cases}
		1 & \text{if } k = j \, , \\
		0 & \text{else,}
	\end{cases}
\end{align*}
where the last equality holds, by Proposition~\ref{prop:step-function-integrand}, because 
\[
\int_{[0,1]} \ybf_k \, \Diff (\indic{\ybf_k > 0})  = 0 \quad \text{and}  \quad \int_{[0,1]} 1 \, \Diff (\indic{\ybf_p > 0})  = 1 \, .
\]
\end{proof}
The following result recovers the scaled coefficient $\wbf_j \xbf_j$ as attribution of feature $j$.
\begin{corollary}[A local feature attribution for linear models] \label{cor:local-linear}
	Set the integrands
	\[
	\forall j \in \llbracket d \rrbracket , \forall \xbf \in (0,1)^d, \forall \ybf \in \Reals^d, \qquad g_{j, \xbf}(\ybf) \defeq  \indic{\ybf_j \geq \xbf_j } \prod_{i \neq j} \indic{\ybf_i > 0 } \, .
	\]
	Let $\phi$ be a \texttt{Linear} and \texttt{FPC} (or \texttt{FSC}) attribution method.
	Assume that the atomic attributions of indicator functions are given by:
	\begin{align*}
		\forall \xbf \in (0,1)^d, \forall j \in \llbracket d \rrbracket, \qquad
		\phi\left(\xbf, \indic{R} \right)_j \defeq 
		\Delta  \left(g_{j, \xbf}, R \right)  \in \Reals \, ,
	\end{align*}
where $R \subset [0,1]^d$ is a right-closed hyperrectangle.
	Then, the attribution of a linear regression model $f_{\wbf}: [0,1]^d \to \Reals$ with parameters $\wbf \in \Reals^d$ is:
	\begin{align*}
		\forall \xbf \in (0,1)^d, \forall j \in \llbracket d \rrbracket, \qquad \phi\left(\xbf, f_{\wbf} \right)_j =  \wbf_j \xbf_j \, .
	\end{align*}
\end{corollary}
\begin{remark}
	Note that we cannot use the indicator function $\indic{\ybf_j \geq 0}$ in the integrands of Corollaries~\ref{cor:global-linear} and~\ref{cor:local-linear}. 
	The issue with using $\indic{\ybf_j \geq 0}$ arises from the boundary of the input space. 
	Specifically, for any function $f: [0,1] \to \Reals$, 
	\[
	\int_{[0,1]} f(y) \, \Diff \indic{y \geq 0} = 0,
	\]
	by Proposition~\ref{prop:step-function-integrand}, since $\indic{y \geq 0} = 1$ on the entire interval $[0,1]$.
	To solve this problem, we should use integrand with strict inequality $\indic{\ybf_j > 0}$. But this integrand lacks right-continuity (which is not problematic in this case to construct the associated measure).
\end{remark}
\begin{proof}
	By Corollary~\ref{cor:linear-model-general}, the attribution of $f_\wbf $ is
	\begin{align*}
		\forall \xbf \in [0,1]^d, \forall j \in \llbracket d \rrbracket, \qquad \phi\left(\xbf, f_{\wbf} \right)_j = \wbf^\top \nbf^{(j, \xbf)} =  \sum_{k \in \llbracket d \rrbracket} \wbf_k \nbf^{(j, \xbf)}_k \, .
	\end{align*}
	We have the following
	\begin{align*}
		\nbf^{(j, \xbf)}_k &\defeq  \int_{[0,1]^d} \ybf_k \, \Diff g_{j, \xbf} (\ybf) \\
		&=  \int_{[0,1]} \int \cdots \int \ybf_k \, \Diff (\indic{\ybf_j \geq \xbf_j}) \bigotimes_{i \neq j} \Diff (\indic{\ybf_i > 0}) &\text{(by Corollary~\ref{cor:multiple-riemann-stieltjes}.)} \\
		&= \begin{cases}
			\left(\prod_{p \neq j} \int_{[0,1]} 1 \, \Diff (\indic{\ybf_p > 0})\right) \int_{[0,1]} \ybf_k \, \Diff (\indic{\ybf_k \geq \xbf_k})  & \text{if } k = j \, , \\
			\int_{[0,1]} \ybf_k \, \Diff (\indic{\ybf_k > 0}) 	\int_{[0,1]} 1 \, \Diff (\indic{\ybf_j \geq \xbf_j})  \prod_{p \neq j,k} \int_{[0,1]} 1 \, \Diff (\indic{\ybf_p > 0}) & \text{else.}
		\end{cases}  &\text{(Fubini's Theorem)}\\
		&= \begin{cases}
			\xbf_k & \text{if } k = j \, , \\
			0 & \text{else,}
		\end{cases}
	\end{align*}
	where the last equality holds, by Proposition~\ref{prop:step-function-integrand}, because
    \[
    \int_{[0,1]} \ybf_k \, \Diff (\indic{\ybf_k \geq \xbf_k})  = \xbf_k \quad \text{and} \quad \int_{[0,1]} \ybf_k \, \Diff (\indic{\ybf_k > 0})  = 0 \, .
    \]
\end{proof}
\subsection{Appendix for Section~\ref{sec:4.3}} \label{app:4.3}
\subsubsection{Casting Problem~\eqref{pb:recall-1} to Problem~\eqref{eq:pb-optim-2}}
In the following section, we explain how the infinite-dimensional optimization Problem~\eqref{pb:recall-1} can be casted into the finite-dimensional Problem~\eqref{eq:pb-optim-2}.
\begin{proposition}[Simple surjection of measure into vectors] \label{prop:surjection}
Set $\mathcal{P}$ as the space of probability measures on $[0,1]^d$.
	There exists a simple surjection from $\mathcal{P} $ to $[0,1]^d$ defined as
	\[
	\mathcal{P} \twoheadrightarrow [0,1]^d : \mu \mapsto \mbf^{\mu} \defeq \left( \int_{[0,1]^d} \ybf_i \Diff \mu(\ybf) \right)_{i \in \llbracket d \rrbracket} \, ,
	\]
	where $\mbf^{\mu}$ is the center of mass of $[0,1]^d$ w.r.t. $\mu$.
\end{proposition}
\begin{proof}
    Given $\mbf \in [0,1]^d$, take the measures $\mu$ on $[0,1]^d$ defined as
    \[
    \forall A \in \mathcal{B}([0,1]^d), \qquad \mu(A) \defeq \bigotimes_{j \in \llbracket d \rrbracket} \delta_{\mbf_j} \, .
    \]
    The center of mass of $[0,1]^d$ w.r.t. $\mu$ is exactly $\mbf$.
\end{proof}
Thanks to the previous proposition, we can reformulate Problem~\eqref{pb:recall-1} as Problem~\eqref{eq:pb-optim-2} by surjection with the center of mass. As noted in the main paper, this reformulation involves a loss of information.
The following technical result guarantees that the \textbf{center of mass stays inside $[0,1]^d$}.
\begin{proposition}[Center of mass stays in convex compact of $\Reals^d$]
	Given a convex, closed and Borel measurable subset $\mathcal{S} \subset [0,1]^d$and a measure $\mu \in \mathcal{P}$. If $\mu(\mathcal{S}) \neq 0$, then the center of mass of $\mathcal{S}$ checks $\mbf^\mu \in \mathcal{S}$.
\end{proposition}
\begin{proof}
	The proof is done by contradiction. Let us assume that $\mbf^{\mu} \notin\mathbf{S}$. By applying the Hahn–Banach separation theorem (with the variant where both sets are closed and at least one of them is compact) on $\{ \mbf^{\mu} \}$ and $\mathcal{S}$, as they do not intersect, there exists a nonzero vector $\vbf \in \Reals^d$ and two real numbers $c_1 > c_2$ s.t. 
	\[
	\forall \xbf \in \Reals^d , \qquad  \ps{\xbf}{\vbf} > c_1 \text{ and } \ps{\mbf^\mu}{\vbf} < c_2			\, .
	\]
	
	Which implies (by $c_1 > c_2$)
	\begin{equation} \label{ine:hahn}
		\forall \xbf \in \mathcal{S} , \qquad  \ps{\xbf}{\vbf} > c_2 > \ps{\mbf^\mu}{\vbf}  \, .
	\end{equation}
	We rewrite the last term of the Inequality~\ref{ine:hahn} as follows
	\begin{align*}
		\ps{\mbf^\mu}{\vbf} &= \ps{ \frac{1}{\mu(\mathcal{S})} \left( \int_{\mathcal{S}} \ybf_1 \, d\mu(\ybf) , \ldots , \int_{\mathcal{S}} \ybf_d \, d\mu(\ybf)  \right) }{\vbf} \\
		&= \frac{1}{\mu(\mathcal{S})} \int_{\mathcal{S}} \ps{\ybf}{\vbf} \, d\mu(\ybf) \, .
	\end{align*}
	
	Now, the contradiction arise by looking at the following integral
	\begin{align*}
	 \frac{1}{\mu(\mathcal{S})}	\int_{\mathcal{S}}  \ps{\ybf}{\vbf} \, d\mu(\ybf) >  \frac{1}{\mu(\mathcal{S})} \int_{\mathcal{S}}  c_2 \, d\mu(\ybf) = c_2 \times 1 > \ps{\mbf^\mu}{\vbf} =  \frac{1}{\mu(\mathcal{S})} \int_{\mathcal{S}} \ps{\ybf}{\vbf} \, d\mu(\ybf) \, ,
	\end{align*}
	where the first and last inequality comes from~\ref{ine:hahn}. To sum up the contradiction is
	\[
	 \frac{1}{\mu(\mathcal{S})}	\int_{\mathcal{S}}  \ps{\ybf}{\vbf} \, d\mu(\ybf) >  \frac{1}{\mu(\mathcal{S})}	\int_{\mathcal{S}}  \ps{\ybf}{\vbf} \, d\mu(\ybf) \, .
	\]
\end{proof}

\subsubsection{Additional results on Precision and ReLU networks}
We introduce the counterpart to \emph{Recall}, called \emph{Precision}.
\begin{definition}[Precision]  
\label{def:formal-prec}
Given a threshold $\beta > 0$ and a linear model $f_\wbf : [0,1]^d \to \Reals$, with weights $\wbf \in \Reals^d$, and the index partition $D_{1, \wbf} \sqcup D_{0, \wbf} = \llbracket d \rrbracket$, where $D_{0, \wbf} \defeq \{ i \in \llbracket d \rrbracket : \abs{\wbf_i} \leq \beta \}$ and $D_{1, \wbf} \defeq D_{0, \wbf}^\complement$ (called the \emph{golden set}). 
    Also, let $\phi$ be a global feature attribution method which depends on a vector of measures $\mu \defeq ( \mu_j )_{j \in \llbracket d \rrbracket} \in \mathcal{P}^d$. 
	We define the \emph{Precision} metric, with a threshold $\alpha > 0$, as
	\[
	 \mathrm{Precision}_{\alpha, \wbf}(\mu) \defeq \frac{\sum_{j \in D_{1, \wbf}} \indic{\abs{\phi(f_\wbf)_j} \geq \alpha}}{\sum_{j \in D_{1, \wbf}} \indic{\abs{\phi(f_\wbf)_j} \geq \alpha} + \sum_{j \in D_{0, \wbf}} \indic{\abs{\phi(f_\wbf)_j} \geq \alpha}}  \, .
        \]
where $\mathcal{P}$ is the space of probability measures on $[0,1]^d$.
\end{definition}
In general, there is a trade-off between \emph{Precision} and \emph{Recall}. 
One easy solution is to prioritize one metric over the other, the choice should depend on whether false positives or false negatives carry a higher cost. Ultimately, this decision depends on the specific requirements of your task, whether you are doing recourse, or generating counterfactual explanations.
We optimize the \emph{Precision} metric using the finite-dimensional formulation
\begin{equation}
\label{pb:precision}
\Argmax_{\mbf^{(1)},  \ldots, \mbf^{(d)} \in  [0,1]^d } \mathrm{Precision}_{\alpha, \wbf}(\mbf^{(1)}, \ldots, \mbf^{(d)})
\, .    
\end{equation}

\begin{theorem}[Optimal projections for $\mathrm{Precision}_{\alpha,\mathbf{w}}$] \label{th:optim-2}
For each $j \in \llbracket d \rrbracket$ define
	\begin{align*}
		\mathcal{S}_j \defeq 
		\begin{cases}
			\left\{\mbf \in [0,1]^d :  \abs{\wbf^\top \mbf} \geq \alpha \right\} \text{  if $j \in D_{1, \wbf}$} \, ,\\
			\left\{\mbf \in [0,1]^d : \abs{\wbf^\top \mbf} < \alpha \right\} \text{  otherwise.}
		\end{cases}
	\end{align*}
Then, the solutions of \eqref{pb:precision} are $\mbf^{(j)}\in \mathcal{S}_j$ for all $j\in \llbracket d\rrbracket$. 
\end{theorem}
\begin{proof}
The goal is to maximize the precision metric $\mathrm{Precision}_{\alpha,\mathbf{w}}$ in the finite-dimensional case.
To do this, we:
\begin{itemize}
	\item \textbf{maximize the numerator:} ensure that $|\phi(f_{\wbf})_j| \geq \alpha$ for all $j \in D_{1,\wbf}$.
	\item \textbf{minimize the denominator:} ensure that $|\phi(f_{\wbf})_j| < \alpha$ for all $j \in D_{0,\wbf}$, but keeping the fact that $|\phi(f_{\wbf})_j| \geq \alpha$ for all $j \in D_{1,\wbf}$.
\end{itemize}

Since the optimization problem involves disjoint index sets ($D_{1,\wbf} \cap D_{0,\wbf} = \varnothing$), the optimization decouples across features $j$.
Note that $\phi(f_{\wbf})_j = \wbf^\top \mbf^{(j)}$, and the indicator $\indic{|\phi(f_{\wbf})_j| \geq \alpha}$ depends only on $\mbf^{(j)}$. Thus, for each $j$:
\begin{itemize}
	\item If $j \in D_{1, \wbf}$, we want $|\wbf^\top \mbf^{(j)}| \geq \alpha$.
	\item If $j \in D_{0, \wbf}$, we want $|\wbf^\top \mbf^{(j)}| < \alpha$.
\end{itemize}
Therefore, for each $j \in \llbracket d \rrbracket$, the optimal $\mbf^{(j)}$ must lie in the set $\mathcal{S}_j$ defined above.
\end{proof}
This result is illustrated in Figure~\ref{fig:regression-optimal-measures-1}, where the white area represent the solutions $S_j$ when $j \in D_{0, \wbf}$.
Now, we show that maximizing \emph{Precision} implies maximizing the \emph{Recall}.
\begin{lemma}[Maximizing Precision $\implies$ maximizing Recall]
	In the previous notations, we have
	\begin{align*}
	    	 &\Argmax_{\mbf^{(1)}, \mbf^{(2)}, \ldots, \mbf^{(d)} \in  [0,1]^d } \mathrm{Precision}_{\alpha, \wbf}(\mbf^{(1)}, \mbf^{(2)}, \ldots, \mbf^{(d)}) \\
     &\subset  \Argmax_{\mbf^{(1)}, \mbf^{(2)}, \ldots, \mbf^{(d)} \in  [0,1]^d } \mathrm{Recall}_{\alpha, \wbf}(\mbf^{(1)}, \mbf^{(2)}, \ldots, \mbf^{(d)}) \, .
	\end{align*}
\end{lemma}
\begin{proof}
This follows directly from Theorem~\ref{th:optim-2}. For all $j \in D_{1,\wbf}$, both optimization problems (for \emph{Precision} and \emph{Recall}) require that $\mbf^{(j)} \in \mathcal{S}_j = \{ \mbf \in [0,1]^d : |\wbf^\top \mbf| \geq \alpha \}$. 
However, for $j \in D_{0,\wbf}$, the recall objective imposes no constraint on $\mbf^{(j)}$, while the precision objective requires $\mbf^{(j)} \in \mathcal{S}_j = \{ \mbf \in [0,1]^d : |\wbf^\top \mbf| < \alpha \}$.
Thus, any solution that maximizes \emph{Precision} also maximizes \emph{Recall}, implying the inclusion.
\end{proof}

Since the local linear models of a ReLU network are explicitly known, one may consider them as a form of ground-truth for feature attribution. This motivates the following definitions of \emph{Recall} and \emph{Precision} tailored to deep ReLU networks.
\begin{definition}[ReLU Recall and Precision]  \label{def:relu-recall-precision}
	Let $f : [0,1]^d \to \Reals$ be a deep ReLU network and $\phi$ be a global feature attribution method which depends on a vector of measures $\mu \defeq ( \mu_j )_{j \in \llbracket d \rrbracket} \in \mathcal{P}^d$. 
	We define the ReLU-Recall metric, with a threshold $\alpha > 0$, as
	\[
	\mathrm{RRecall}_{\alpha, f}(\mu) \defeq  \sum_{P \in \mathcal{R}([0,1]^d)} \mathrm{Recall}_{\alpha, \abf_P} (\mu)   \, ,
	\]
	and ReLU-Precision as
	\[
	  \mathrm{RPrecision}_{\alpha, f}(\mu) \defeq  \sum_{P \in \mathcal{R}([0,1]^d)} \mathrm{Precision}_{\alpha, \abf_P} (\mu)    \, .
	\]
     For each region $P$, the corresponding local linear model has coefficients $(\mathbf{a}_P, b_P)$ (see Property~\ref{th:rep-relu}).
\end{definition}
This definition extends \emph{Recall} and \emph{Precision}, from Definition~\ref{def:formal-recall}, to ReLU networks.
Again, we optimize the ReLU-\emph{Recall} metric using the finite-dimensional formulation
\begin{equation}
\label{pb:rrecall}
\Argmax_{\mbf^{(1)},  \ldots, \mbf^{(d)} \in  [0,1]^d } \mathrm{RRecall}_{\alpha, f}(\mbf^{(1)}, \ldots, \mbf^{(d)})
\, .    
\end{equation}

\begin{theorem}[Optimal projections for $\mathrm{RRecall}_{\alpha, f}$] \label{th:optim-rrecall}
For each $j \in \llbracket d \rrbracket$ define
	\begin{align*}
		\mathcal{S}_j \defeq 
		\bigcap_{P \in \mathcal{R}([0,1]^d) \text{ s.t. } j \in D_{1, \abf_P}}	\left\{\mbf \in [0,1]^d : \abs{\abf_P^\top \mbf} \geq \alpha \right\} \, .
	\end{align*}
Then, for Problem~\eqref{pb:rrecall}, a subset of solutions is given by choosing $\mathbf{m}^{(j)} \in \mathcal{S}_j$ for all $j \in \llbracket d \rrbracket$.
\end{theorem}
\begin{proof}
We want to maximize the objective $\mathrm{RRecall}_{\alpha, f}$, which aggregates recall values over all regions $P \in \mathcal{R}([0,1]^d)$.
A sufficient condition for maximizing $\mathrm{RRecall}_{\alpha, f}$ is to simultaneously maximize each individual $\mathrm{Recall}_{\alpha, \abf_P}$. 
By Theorem~\ref{th:optim-1}, for each $P \in \mathcal{R}([0,1]^d)$ and each $j \in D_{1, \abf_P}$, recall is maximized when $\abs{\abf_P^\top \mbf^{(j)}} \geq \alpha$. Since the same feature $j$ may belong to $D_{1,\abf_P}$ for multiple $P$, we intersect the corresponding solutions sets to find a common region of optimality.
Therefore, for each $j \in \llbracket d \rrbracket$, the set of $\mbf^{(j)}$ that simultaneously maximizes all relevant $\mathrm{Recall}_{\alpha, \abf_P}$ is
\begin{align*}
    \mathcal{S}_j \defeq 
    \bigcap_{P \in \mathcal{R}([0,1]^d) \text{ s.t. } j \in D_{1, \abf_P}}	\left\{\mbf \in [0,1]^d : \abs{\abf_P^\top \mbf} \geq \alpha \right\} \, .
\end{align*}
Note that this characterization may not capture all maximizers, as the constraints are not disjoint and the intersection may be empty.
\end{proof}
We obtain a similar result for the optimization of ReLU-\emph{Precision},
\begin{equation}
\label{pb:rprecision}
\Argmax_{\mbf^{(1)},  \ldots, \mbf^{(d)} \in  [0,1]^d } \mathrm{RPrecision}_{\alpha, f}(\mbf^{(1)}, \ldots, \mbf^{(d)})
\, .    
\end{equation}

\begin{theorem}[Optimal projections for $\mathrm{RPrecision}_{\alpha, f}$] \label{th:optim-rprecision}
For each $j \in \llbracket d \rrbracket$ define
	\begin{align*}
		\mathcal{S}_j \defeq \bigcap_{P \in \mathcal{R}([0,1]^d)} \mathcal{S}_{j, P} \, ,
	\end{align*}
    where $S_{j,P}$ is
    \begin{align*}
    \mathcal{S}_{j,P} \defeq 
    \begin{cases}
        \left\{\mbf \in [0,1]^d :  \abs{\abf_P^\top \mbf} \geq \alpha \right\} \text{  if $j \in D_{1, \abf_P}$} \, ,\\
        \left\{\mbf \in [0,1]^d : \abs{\abf_P^\top \mbf} < \alpha \right\} \text{  otherwise.}
    \end{cases}
\end{align*}
Then, for Problem~\eqref{pb:rprecision}, a subset of solutions is given by choosing $\mathbf{m}^{(j)} \in \mathcal{S}_j$ for all $j \in \llbracket d \rrbracket$.
\end{theorem}
\begin{proof}
This is exactly the same proof as for Theorem~\ref{th:optim-rrecall}, but we use Theorem~\ref{th:optim-2} (instead of Theorem~\ref{th:optim-1}) as we are maximizing $\mathrm{Precision}_{\alpha, \abf_P}$.    
\end{proof}

\section{Proofs of main results}
\label{app:proofs-main}

\subsection{Proof of Proposition~\ref{prop:sens_complete}}

Let $f :  \Reals \to \Reals, \xbf \in \Reals^d$ and $k,j \in \llbracket d \rrbracket$. Since the function $f\circ \pi_j$ depends only on the $j$-th coordinate of $\xbf$ and remains unchanged when any other coordinates varies, it follows from the \texttt{Sensitivity} property that $\phi\left(\xbf, f\circ \pi_j \right)_k = 0$ for $k \neq j$.

Now, by \texttt{Completude},
\[
f(\xbf_j) - f(\xbf_j') = \sum_{p \in \llbracket d \rrbracket}   \phi\left(\xbf, f\circ \pi_j \right)_p =   \phi\left(\xbf, f\circ \pi_j \right)_j \, .
\]
Finally,
\[
\phi\left(\xbf, f\circ \pi_j \right) = \left(0, \ldots,  f(\xbf_j) - f(\xbf_j') , \ldots, 0 \right)^\top  \, .
\]
\subsection{Proof of Lemma~\ref{lemma:taylor-phi}}

Given $f \in \mathcal{F}$ and $\xbf_0 \in \Reals^d$, we apply Taylor-Lagrange theorem~\cite{courant2000introduction} on $f$ at $\xbf_0$:
\[
\forall \xbf \in \Reals^d, \qquad f(\xbf) = f(\xbf_0) + \nabla f(\xbf_0)^\top (\xbf - \xbf_0) + R_{\xbf_0}(\xbf)           \, .
\]

Now, by \texttt{Linearity} property, 
\[
\forall \xbf \in \Reals^d, \qquad \phi(\xbf, f) = \phi \left(\xbf, f(\xbf_0) \right) +  \phi \left(\xbf, \nabla f(\xbf_0)^\top (\cdot - \xbf_0) \right) +  \phi( \xbf, R_{\xbf_0}) \, .
\]
Given $\xbf \in \Reals^d$, we tackle the zero-order term as follows:
\begin{align*}
    \phi \left(\xbf, f(\xbf_0) \right) &= 0 \, . &&\text{(by \texttt{Sensibility})} \\
    \intertext{For the first-order term}
    \phi \left(\xbf, \nabla f(\xbf_0)^\top (\cdot - \xbf_0) \right) &=  \phi \left( \xbf, \nabla f(\xbf_0)^\top (\cdot) \right) - \phi \left(\xbf , \nabla f(\xbf_0)^\top \xbf_0 \right) &&\text{(by \texttt{Linearity})} \\
    &= \phi \left(\xbf, \nabla f(\xbf_0)^\top (\cdot) \right) &&\text{(by \texttt{Sensibility})} \\
    &= \phi \left(\xbf, \sum_{j \in \llbracket d \rrbracket} \pdv{f}{\xbf_j}(\xbf_0) \times \pi_j \right) \\
    &=  \sum_{j \in \llbracket d \rrbracket} \pdv{f}{\xbf_j}(\xbf_0)   \phi(\xbf, \pi_j)   &&\text{(by \texttt{Linearity})} \\
    &=   \sum_{j \in \llbracket d \rrbracket} \pdv{f}{\xbf_j}(\xbf_0)   \left(0, \ldots,  \xbf_j  - \xbf_j', \ldots, 0 \right)^\top      &&\text{(by Proposition~\ref{prop:sens_complete})} \\
    &= \nabla f (\xbf_0) \odot (\xbf - \xbf' ) \, .
\end{align*}

\subsection{Proof of Theorem~\ref{th:bound-1}}

Given a model $f : [0,1]^d \to \Reals$ and a base point $\xbf_0 \in [0,1]^d$,  we apply Lipschitz continuity of $\phi$ as follows: 
\begin{align*}
    \forall \xbf \in [0,1]^d, \exists L_\xbf \geq 0, \qquad \norm{\phi(\xbf, R_{\xbf_0}) - \phi(\xbf, 0)}_2 \leq L_\xbf \norm{ R_{\xbf_0}  - 0}_\infty \, ,
\end{align*}
where $R_{\xbf_0}$ comes from Lemma~\ref{lemma:taylor-phi}. 
Applying Lipschitz continuity with $0$ and $R_{\xbf_0}$ is valid as they are both bounded functions on the compact $[0,1]^d$. 
Now, as $f$ is twice continuously differentiable on the compact $[0,1]^d$, its Hessian is bounded at any points in $[0,1]^d$ by $M \geq 0$:
\[
\forall \xbf \in [0,1]^d,  \qquad  \norm{\nabla^2 f (\xbf)}_{\mathrm{op}} \leq M  \, .
\]
Then, we can bound the remainder $R_{\xbf_0}$ as follows:
\begin{align}\label{eq:bound-remainder}
    \begin{split}
        \forall \xbf \in [0,1]^d, \qquad \abs{R_{\xbf_0}(\xbf)} &= \abs{\frac{1}{2}(\xbf - \xbf_0)^\top \nabla^2 f (\ybf_\xbf) (\xbf - \xbf_0)} \\
        &\leq \frac{1}{2} \norm{\nabla^2 f (\ybf_\xbf)}_{\mathrm{op}} \norm{\xbf - \xbf_0}_2^2 \\
        &\leq \frac{M d}{2} \, ,
    \end{split}
\end{align}
where the last inequality comes from the fact that $\norm{\xbf - \xbf_0}_2^2 \leq \left(\sqrt{d} \right)^2$ on $[0,1]^d$, and $\ybf_\xbf$ is defined in Lemma~\ref{lemma:taylor-phi}.
As the bound in Equation~(\ref{eq:bound-remainder}) is independent of $\xbf$, we can pass to the supremum on $[0,1]^d$ as follows:
\begin{align}\label{eq:bound-remainder-2}
    \begin{split}
        \norm{R_{\xbf_0}}_{\infty} \leq  \frac{M d}{2} \, .
    \end{split}
\end{align}
Finally, we have the following bound on the remainder:
\begin{align*}
    \forall \xbf \in [0,1]^d, \exists L_\xbf \geq 0, \qquad \norm{\phi(\xbf, R_{\xbf_0}) - \phi(\xbf, \widebar{0})}_2 &\leq L_\xbf \norm{R_{\xbf_0} - \widebar{0}}_\infty \\
    & \leq \frac{L_\xbf  d}{2} M  &&\text{(by Equation~(\ref{eq:bound-remainder-2}))} \, .
\end{align*}

\subsection{Proof of Theorem~\ref{th:measure-atomic-attribution}} \label{add:proof-measures}
\begin{proof}
We start from Theorem~\ref{th:increment-based-atomic-attribution}, which expresses the attribution in terms of a Riemann–Stieltjes integral. Next, we invoke the correspondence between integrands and signed measures as established in Theorem~\ref{th:signed-borel}. 
Finally, we note that in our setting, the Riemann–Stieltjes and Lebesgue–Stieltjes integrals coincide in value (see~\cite{folland1999real}), thereby concluding the proof.
\end{proof}
\subsection{Proof of Table~\ref{tab:example-feature-attributions}} \label{add:proof-table}
\begin{proof}
See~\cite{folland1999real}. The result follows from the definition of the expectation as integration with respect to the associated probability measure.
\end{proof}

\subsection{Proof of Corollary~\ref{cor:linear-positive}} \label{add:proof-linear-positive}
\begin{proof}
	The proof is a slight modification of the one for Corollary~\ref{cor:relu-attrib}. We apply Corollary~\ref{cor:positive-borel} to get \textbf{positive} Borel measures $\{\mu_{j, \xbf}\}_j$. Within the setting of Corollary~\ref{cor:relu-attrib}, we have
	\begin{align*}
			\forall \xbf \in [0,1]^d, \forall j \in \llbracket d \rrbracket, \qquad \phi\left(\xbf, f \right)_j &= \sum_{P \in \mathcal{R}_{0,j}^\complement} \mu_{j, \xbf} (P) \left( \abf_{P}^\top \mathbf{m}^{(P, j, \xbf)} + b_P \right) + 
		\sum_{P \in \mathcal{R}_{0,j}} \abf_{P}^\top \mathbf{n}^{(P, j, \xbf)} \, .\\
		\intertext{If $\mu_{j, \xbf} (P) = 0$, then $\mathbf{n}^{(P, j, \xbf)}_i \defeq \int_{[0,1]^d} \ybf_i  \, \Diff \mu_{j, \xbf}(\ybf) = 0$ (which is not true for a signed measure)}
		\phi\left(\xbf, f \right)_j  &= \sum_{P \in \mathcal{R}_{0,j}^\complement} \mu_{j, \xbf} (P) \left( \abf_{P}^\top \mathbf{m}^{(P, j, \xbf)} + b_P \right) + 0 \, .\\
		\intertext{We change the definition of $\mathbf{m}^{(P, j, \xbf)}_i$ to}
		\mathbf{m}^{(P, j, \xbf)}_i &\defeq 
		\begin{cases}
			\frac{ \int_{P} \ybf_i \, \Diff \mu_{j, \xbf} (\ybf)}{\mu_{j, \xbf}(P)}  & \text{if $\mu_{j, \xbf}(P) \neq 0$} \, ,\\
			0 & \text{else.}
		\end{cases} \\
		\intertext{Finally, we get}
		\phi\left(\xbf, f \right)_j   &= \sum_{P \in \mathcal{R}([0,1]^d)} \mu_{j, \xbf} (P) \left( \abf_{P}^\top \mathbf{m}^{(P, j, \xbf)} + b_P \right) \, .
	\end{align*}
\end{proof}

\subsection{Proof of Theorem~\ref{th:optim-1}} \label{add:proof-optim1}
\begin{proof}
The goal is to maximize $\mathrm{Recall}_{\alpha, \wbf}(\mbf^{(1)}, \ldots, \mbf^{(d)})$.
As stated in the main paper, the optimization problem is disjoint over the arguments $\mbf^{(j)}$, since each term of the sum $\indic{\abs{\phi(f_\wbf)_j} \geq \alpha}$ depends only on $\mbf^{(j)}$. Therefore, each feature $j \in \llbracket d \rrbracket$ can be optimized independently.
To maximize the sum defining \emph{Recall}, we need $\indic{\abs{\phi(f_\wbf)_j} \geq \alpha} = 1$ for all $j \in D_{1,\wbf}$, which is achieved whenever $\abs{\wbf^\top \mbf^{(j)}} \geq \alpha$ (see Corollary~\ref{cor:global-linear} for the attribution of linear models).
This leads to the following definition of the solution sets:
\[
\mathcal{S}_j \defeq 
\begin{cases}
	\left\{ \mbf \in [0,1]^d : \wbf^\top \mbf \geq \alpha \right\} \cup \left\{ \mbf \in [0,1]^d : \wbf^\top \mbf \leq -\alpha \right\} & \text{if } j \in D_{1, \wbf}, \\
	[0,1]^d & \text{otherwise}.
\end{cases}
\]
\end{proof}

\end{document}